\documentclass[]{fairmeta}
\usepackage{lineno}
\usepackage{calc}
\usepackage{amsmath}
\usepackage{caption}
\usepackage{multirow}
\usepackage{amsfonts}
\usepackage{xspace}
\usepackage[font=footnotesize]{subcaption}
\usepackage{booktabs}
\usepackage{pdfpages}
\usepackage{xcolor}
\usepackage{tabularx}
\usepackage{bbm}
\usepackage{graphicx}
\usepackage{algorithm}
\usepackage{algorithmic}
\usepackage{enumitem}
\usepackage{pifont}

\usepackage{amsmath}
\usepackage{amssymb}
\usepackage{mathtools}
\usepackage{amsthm}
\usepackage{xspace}
\usepackage{algorithm}
\usepackage{algorithmic}
\usepackage{wrapfig}
\usepackage{enumitem}
\usepackage{ulem}

\usepackage[export]{adjustbox}

\newcommand{\ours}{\textsc{OmniOPD}\xspace}

\newcommand{\tabref}[1]{Table~\ref{#1}}
\newcommand{\figref}[1]{Fig.~\ref{#1}}
\newcommand{\eqnref}[1]{\text{Eq.}~(\ref{#1})}
\newcommand{\secref}[1]{\S\ref{#1}}
\newcommand{\appref}[1]{Appendix~\ref{#1}}
\newcommand{\theoremref}[1]{Theorem~\ref{#1}}

\theoremstyle{plain}
\newtheorem{theorem}{Theorem}[section]

\theoremstyle{definition}

\theoremstyle{remark}
\newtheorem{remark}[theorem]{Remark}

\title{\ours: Logit-Free On-Policy Distillation via Speculative Verification}

\author{Yuhang Zhou}
\author{Lizhu Zhang}
\author{Yifan Wu}
\author{Mingyi Wang}
\author{Bo Peng}
\author{Jiayi Liu}
\author{Xiangjun Fan}
\author{Zhuokai Zhao}

\affiliation{Meta AI}

\abstract{On-Policy Distillation (OPD) trains a student model on its own generative trajectories under dense token-level feedback from a stronger teacher, mitigating both the off-policy distribution shift of Supervised Fine-Tuning (SFT) and the sparse credit assignment of outcome-based Reinforcement Learning (RL).
However, standard OPD faces two coupled limitations. 
First, it requires direct access to the teacher's token-level logits, excluding a broad class of capable proprietary models from serving as teachers. 
Second, the token-level logit signal itself is brittle, depending on a narrow overlap of plausible next tokens between teacher and student, and prone to amplifying degenerate patterns such as repetition loops.
In this paper, we introduce \textbf{\ours}, a novel framework that addresses both limitations through a logit-free, chunk-level supervision signal.
\ours replaces deterministic logit matching with Monte Carlo rollouts that approximate the teacher's local preferences through a continuous semantic similarity metric over multi-token chunks, and concentrates this supervision via a peak-entropy scheduler that audits the student only at its high-uncertainty reasoning forks.
A Dirichlet-Multinomial Bayesian prior and a base-model KL anchor further bound the variance of discrete sampling and prevent policy collapse across unaudited tokens.
Across mathematical reasoning and competitive programming benchmarks, \ours outperforms SFT by up to +45.31\% relative on math and +18.52\% relative on code, and surpasses standard white-box OPD approach by up to +28.64\% on mathematical reasoning, confirming that chunk-level semantic verification extracts a cleaner and more reliable learning signal than token-level logit matching, whose high information density is offset by significant noise and brittleness.
Furthermore, when paired with stronger black-box teachers such as Claude-4.5-Haiku and Gemini-2.5-Flash, \ours achieves an additional +9.54\% relative on mathematical reasoning over its open-weight teacher counterpart, advancing the student past the performance of self-exploratory RL.}

\date{\today}
\correspondence{Yuhang Zhou and Zhuokai Zhao at \email{\{zyhang, zhuokai\}@meta.com}\\}

\begin{document}

\maketitle

\section{Introduction}
\label{sec:introduction}

On-Policy Distillation (OPD) has emerged as a central paradigm for post-training open-weight reasoning models~\citep{agarwal2024policy, lu2025onpolicydistillation, song2026survey}.
By having the student generate its own reasoning trajectories and querying a stronger teacher to score them, OPD combines the dense token-level supervision of Supervised Fine-Tuning (SFT) with the on-policy exploration of Reinforcement Learning (RL), mitigating both the distribution shift of imitating offline demonstrations~\citep{huan2025does} and the sparse credit assignment of outcome-based RL methods such as GRPO~\citep{shao2024deepseekmath, chen2025exploration}.

Despite these strengths, standard OPD inherits two coupled limitations that both stem from using the teacher's token-level log-probabilities as the supervision signal.
First, standard OPD approaches~\citep{agarwal2024policy, lu2025onpolicydistillation} requires direct read-out of the teacher's exact next-token logits. 
This limitation is increasingly restrictive, as a broad class of capable proprietary models, including GPT~\citep{singh2025openai}, Claude~\citep{anthropic2025system}, Gemini~\citep{pichai2025new}, and the wider commercial LLM ecosystem~\citep{aubakirova2026state}, return only text, not logits, and are therefore excluded from the OPD pipeline entirely~\citep{song2026survey}.

Second, even when full logit access is available, token-level matching is a surprisingly brittle target. 
Recent studies of OPD failure modes~\citep{li2026rethinking, luo2026demystifying, fu2026revisiting} show that the effective supervision signal lives in the narrow overlap between the teacher's and student's plausible next-token sets --- a small set of plausible next tokens that concentrates the bulk of the gradient mass.
This narrow region is fragile: it degrades when teacher and student reason in stylistically different ways, becomes uninformative or actively harmful on degenerate prefixes such as repetition loops, and is further disrupted by tokenizer mismatches between proprietary and open-weight model families~\citep{fu2026revisiting}.
The high-dimensional logit distribution offers far less reliable supervision than its information density would suggest.
This motivates our central question: 
\begin{quote}
    \centering
    \textit{Can dense on-policy distillation be performed without access to the teacher's logits, and what alternative dense supervision signal could match, or surpass, logit-based matching?}
\end{quote}

We answer both questions with \textbf{\ours}, a logit-free framework that supervises the student through a multi-token, semantic-level signal estimated from text alone.
Our approach builds on \textit{Monte Carlo estimation}~\citep{kroese2014monte}. 
Rather than requiring deterministic logit distributions, \ours approximates the teacher's local preferences by sampling multiple discrete rollouts and estimating the teacher's \textit{true} probability density from their semantic similarity to the student's generation. 
To make this computationally viable under restricted query budgets, we propose a peak-entropy \textit{chunk scheduling}. 
Instead of querying the teacher at every token, \ours selectively audits chunk of student tokens at the high-uncertainty decision boundaries in the student's reasoning trajectory, which is the point where corrective supervision is most valuable.
Finally, to counteract the high variance and zero-probability traps inherent in discrete sampling, we introduce a \textit{Bayesian smoothing step}~\citep{robbins1992empirical}. 
By anchoring the empirical Monte Carlo estimates with the student's base-model prior, \ours extracts a variance-reduced gradient signal, while a separate base-model KL anchor on un-audited tokens prevents policy collapse across the trajectory's unsupervised gaps.

Across mathematical reasoning and competitive programming benchmarks, \ours substantially outperforms SFT and, more strikingly, frequently outperforms white-box OPD~\citep{lu2025onpolicydistillation} even when full teacher logits are available. 
This finding provides empirical support for our second motivation, where a coarser but cleaner signal yields stronger gradients than the high-bandwidth but noisy token-level distribution that white-box OPD optimizes.
When paired with stronger black-box teachers such as Claude-4.5-haiku~\citep{anthropic2025system} and Gemini-2.5-flash~\citep{comanici2025gemini}, \ours further advances the student past the performance ceiling of self-exploratory GRPO~\citep{shao2024deepseekmath} on the same model. 
%
% Extensive evaluations across complex mathematical reasoning and competitive programming benchmarks validate the efficacy of our framework. 
%
% \ours achieves substantial performance gains over standard SFT across various open-weight and closed-source teacher setups, delivering up to +7.40\% gain over offline SFT on mathematical reasoning with Claude-4.5-haiku as the teacher. 
%
% Remarkably, \ours frequently surpasses full-logit white-box OPD by granting the student localized semantic flexibility rather than enforcing strict token-by-token imitation. 
%
% Furthermore, when leveraging frontier models such as Claude and Gemini as teachers, \ours decisively breaks the performance ceilings of standard GRPO empowering smaller student models to achieve extraordinary reasoning capabilities.

In summary, our main contributions are as follows:
\begin{itemize}
    % \item We formalize the problem of black-box on-policy distillation and introduce \ours, a framework that successfully bypasses the white-box logit bottleneck via semantic Monte Carlo estimation.
    \item We propose \ours, a logit-free, chunk-level OPD framework that addresses both the white-box access requirement of standard OPD and the brittleness of its token-level supervision signal.
    %
    % \item We design an efficient peak-entropy scheduling mechanism that concentrates teacher verification at the trajectory's most uncertain reasoning forks, coupled with a Bayesian smoothing mechanism and trust-region penalty to mathematically stabilize discrete sampled targets and prevent catastrophic policy collapse.
    \item We introduce three components that make \ours tractable and stable: a peak-entropy scheduler that concentrates teacher verification at the trajectory's most uncertain reasoning forks, a Dirichlet-Multinomial Bayesian prior that stabilizes the discrete Monte Carlo estimates, and a trust-region KL penalty that prevents policy collapse across un-audited tokens. 
    We provide formal guarantees on estimation error, gradient variance, and trajectory stability.
    \item We demonstrate empirically that \ours outperforms SFT by up to +45.31\% relative on mathematical reasoning and +18.52\% on competitive programming, exceeds white-box OPD by up to +28.64\% on mathematical reasoning, and, when paired with stronger black-box teachers such as Claude-4.5-haiku and Gemini-2.5-flash, advances the student past the performance ceiling of self-exploratory GRPO.
    % \item We demonstrate empirically that \ours consistently outperforms SFT, matches or exceeds white-box OPD (e.g., +4.92\% on average in mathematical reasoning), and seamlessly scales to frontier commercial systems to surpass standard RL limits.
\end{itemize}
\section{Related Work}
\label{sec:related_work}

\paragraph{Offline policy distillation of LLMs.}
The traditional paradigm for transferring complex reasoning capabilities from frontier models to smaller, more deployable models is offline knowledge distillation~\citep{hsieh2023distilling, shridhar2023distilling, kim2016sequence, hao2026self}. 
Early approaches primarily relied on sequence-level distillation from black-box APIs, querying more capable models to generate high-quality Chain-of-Thought (CoT)~\citep{wei2022chain} trajectories and training the student model to mimic these sequences via SFT~\citep{ho2022large, zhou2024teaching, lee2024mentor}. 
While effective for basic instruction following, this sequence-level imitation often fails to capture the nuanced, token-by-token reasoning dynamics of the teacher.
%
% \zz{
With the emergence of highly capable open-source models~\citep{grattafiori2024llama, abdin2024phi, yang2025qwen3, guo2025deepseek, kamath2025gemma, team2025kimi, xiao2026mimo}, the field increasingly shifted toward more fine-grained, white-box distillation techniques that directly optimize the student against the teacher's continuous output distributions via Forward Kullback-Leibler (KL) divergence or token-level logit matching~\citep{gu2024minillm, wen2023f}.
% }
%
However, whether utilizing black-box textual trajectories or white-box logits, these approaches are fundamentally off-policy: the student is trained on the teacher's static state distribution, leading to inevitable distribution shift and exposure bias during deployment.

\paragraph{On-policy distillation and its limitations.}
To mitigate the exposure bias inherent to offline methods, OPD shifts the learning objective to the student's own generated distribution~\citep{agarwal2024policy, lu2025onpolicydistillation}. 
In standard OPD frameworks, the student model actively generates reasoning trajectories that are then evaluated and scored by the teacher model, bridging the gap between training and inference distributions. 
Recent variants have expanded this paradigm to self-OPD, where a model iteratively improves through its own on-policy feedback~\citep{zhao2026self, he2026self, hubotter2026reinforcement}, as well as generalized OPD frameworks designed to handle architectural differences such as mismatched tokenizers between the student and teacher model~\citep{patino2025_unlocking}.
%
% \zz{
Despite these advancements, recent analyses have revealed that the token-level logit signal underlying standard OPD is surprisingly brittle. 
\citet{li2026rethinking} show that the effective supervision signal concentrates in the narrow overlap between teacher and student plausible next-token sets, and that a teacher with strong benchmark accuracy can fail to improve a student whose reasoning trajectories fall outside this overlap. 
\citet{luo2026demystifying} document a systematic collapse mode in which degenerate prefixes such as repetition loops accidentally receive large positive gradient signal, since locally predictable tokens are assigned high probability by the teacher. 
\citet{fu2026revisiting} further show that single-sampled-token feedback is statistically noisy, and that tokenizer mismatches between teacher and student introduce a class of spurious disagreements that no amount of logit access can resolve. 
These findings establish token-level logit matching as a structurally fragile supervision signal, motivating the chunk-level semantic approach we develop in this paper.
Beyond brittleness, standard OPD also faces a structural access barrier that computing the Reverse KL divergence between student and teacher requires explicit access to the teacher's token probabilities~\citep{yang2026learning}. 
Consequently, OPD with proprietary teachers has remained largely unexplored. 
\citet{ye2025black} recently introduced Generative Adversarial Distillation (GAD), which trains an auxiliary discriminator as an evolving reward model in a minimax game with the student.
While GAD bypasses the logit requirement, its discriminator yields only a sequence-level scalar score, providing coarser feedback than the per-step supervision of standard OPD, and requiring an additional model to train and maintain. 
\ours instead recovers dense, chunk-level supervision directly from teacher rollouts through Monte Carlo semantic estimation, without any auxiliary network, and matches or exceeds white-box logit matching when full teacher access is available.
%
% }

% Despite these advancements, standard OPD fundamentally relies on white-box teachers, as computing the Reverse KL divergence requires explicit access to the teacher's token probabilities \citep{yang2026learning}. Consequently, OPD with frontier black-box models has remained largely unexplored. While \citet{ye2025black} explores online distillation from black-box models via adversarial methods, such approaches necessitate training an auxiliary reward model and only yield sparse scalar feedback, lacking the dense optimization signals characteristic of standard OPD. 

% \ours is the first framework to bridge the gap between rigorous OPD and black-box distillation without relying on adversarial mechanisms. By mathematically estimating the token probabilities via Bayesian-smoothed Monte Carlo rollouts, our method introduces a stable, probabilistically grounded pathway to distill the active reasoning capabilities of frontier black-box models.
\section{Methodology}
\label{sec:workflow}

\subsection{Preliminaries}
\paragraph{Offline Supervised Fine-Tuning (SFT).}
The standard paradigm for transferring reasoning capabilities from a strong teacher model to a smaller student model is SFT. 
Given a dataset of instructions $x$ and corresponding reasoning trajectories $y = (y_1, y_2, \dots, y_T)$ generated by a teacher, SFT optimizes the student policy $\pi_\theta$ via maximum likelihood estimation. 
The objective minimizes the standard cross-entropy loss:
\begin{equation}
    \mathcal{L}_{\text{SFT}}(\theta) = - \mathbb{E}_{(x, y) \sim \mathcal{D}} \left[ \sum_{t=1}^{T} \log \pi_\theta(y_t \mid x, y_{<t}) \right]
\end{equation}
While efficient, offline SFT inherently suffers from exposure bias and distribution shift, as the student is trained on the teacher's static state distribution, not its own~\citep{lu2025onpolicydistillation}.

\paragraph{Online On-Policy Distillation (OPD).}
% To mitigate the distribution shift of SFT, OPD shifts the learning paradigm from imitation to active exploration. 
% \zz{
OPD addresses this exposure bias by training on the student's own generations rather than the teacher's static trajectories.
Let $\hat{y} \sim \pi_\theta(\cdot \mid x)$ be a trajectory sampled from the student policy.
Standard OPD minimizes a per-token divergence (most commonly the Reverse KL divergence) between the student and the true teacher policy $\pi_{teacher}^*$~\citep{agarwal2024policy, lu2025onpolicydistillation}:
%
% }
%
% In OPD, the student model generates its own reasoning trajectory $\hat{y} \sim \pi_\theta(\cdot \mid x)$ (i.e., operating on-policy). 
%
% The teacher model then verifies this student-generated trajectory to provide a learning signal. 
%
% Using our established notation, standard OPD minimizes a divergence—typically the Reverse KL divergence between the student's policy $\pi_\theta$ and the true, exact teacher policy $\pi_{teacher}^*$:
\begin{equation}
    \mathcal{L}_{\text{OPD}}(\theta) = \mathbb{E}_{x \sim \mathcal{D}, \hat{y} \sim \pi_\theta} \left[ \sum_{t=1}^{|\hat{y}|} D_{\text{KL}} \Big( \pi_\theta(\cdot \mid x, \hat{y}_{<t}) \parallel \pi_{\text{teacher}}^*(\cdot \mid x, \hat{y}_{<t}) \Big) \right]
    \label{eq:opd_loss}
\end{equation}
By exploring its own action space and receiving corrective feedback from the teacher distribution $\pi_{teacher}^*$, the student learns a more robust policy that bridges the gap between training and inference distributions.
%

% \paragraph{The White-Box Bottleneck.}
% A critical limitation of current OPD methodologies is their strict reliance on \textit{white-box} teacher models. 
% %
% To compute the exact KL divergence for optimization, the objective requires direct access to the teacher's continuous token-level log-probabilities, $\log \pi_{teacher}^*(y_t \mid x, y_{<t})$. 
% %
% Consequently, standard OPD frameworks cannot utilize frontier black-box models as teachers because these black-box systems do not output the required token-level logits.

\subsection{\ours}
% \zz{WIP}

\begin{figure*}[t]
    \centering
    \includegraphics[width=\textwidth]{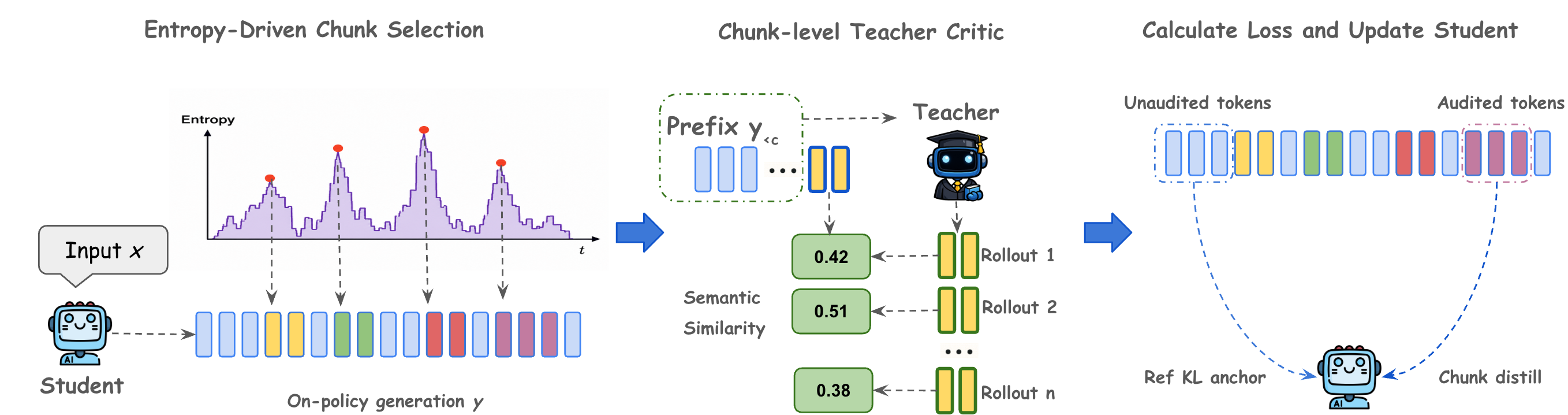}
    \caption{
        Overview of \ours. 
        \textbf{Entropy-Driven Chunk Selection:} 
        Given an input prompt $x$, the student model generates an on-policy trajectory. 
        \ours first identifies critical reasoning forks by locating local peaks in predictive entropy, selectively auditing only these high-uncertainty chunks. 
        \textbf{Chunk-level Teacher Critic:} 
        For each selected chunk, the trajectory prefix $y_{<c}$ is passed to the teacher model. 
        The teacher generates $N$ Monte Carlo rollouts, which are then evaluated against the student's generated chunk via a continuous semantic similarity metric $\phi$. 
        \textbf{Calculate Loss and Update Student:} 
        The student policy is updated by applying the distillation loss exclusively to the audited tokens. 
        Simultaneously, a reference KL anchor is applied across the unaudited tokens to prevent policy shift.
        %
        % \zz{The three component names on top are not aligned horizontally. Also we can change ``closed-source teacher'' to just ``Teacher Model'' as we expanded the claim?}
        % \zz{first column text higher than the other two}
    }
    \label{fig:omniopd_framework}
\end{figure*}

% To bypass the white-box bottleneck, a natural approach is to approximate the true teacher policy $\pi_{teacher}^*$ using only the discrete textual outputs via a Monte Carlo (MC) method. 
%
% At every generation step $t$, one could theoretically prompt the teacher to generate multiple independent rollouts and calculate the empirical frequency of the student's exact next token. 
%
% However, executing this naive token-by-token estimation across an entire reasoning trajectory of length $T$ is economically intractable due to the massive volume of teacher queries required. 
%
% Furthermore, it is statistically brittle; vocabulary and tokenization mismatches between distinct models make exact token-level matches highly improbable, yielding sparse, zero-reward gradients. 
%
% \zz{
\ours addresses both limitations of standard OPD identified in \secref{sec:introduction} through replacing the token-level logit signal with a chunk-level semantic signal estimated from text along.
As illustrated in \figref{fig:omniopd_framework}, \ours has three components: a peak-entropy scheduler that selects high-uncertainty chunks from the student's trajectory (\secref{subsubsec:chunk_selection}), a chunk-level teacher critic that estimates the teacher's preference through Monte Carlo rollouts evaluated by semantic similarity (\secref{subsubsec:speculative_verification} and~\ref{subsubsec:bayesian_smoothing}), and a trust-region anchor that prevents policy drift on un-audited tokens (\secref{subsubsec:trust_region}).
The full algorithm is presented in Algorithm~\ref{alg:OPD} in \appref{app:supp_omniopd}.
%
% }
% To resolve these dual bottlenecks, as illustrated in Figure \ref{fig:omniopd_framework}, \ours introduces a progressive, chunk-level formulation bounded by Bayesian priors and trust regions. The full algorithm is presented in Algorithm \ref{alg:OPD} in Appendix.

\subsubsection{Chunk-Level Distillation via Speculative Verification}
\label{subsubsec:speculative_verification}
% \zz{
Our logit-free distillation is built on Monte Carlo estimation of the teacher's preference.
A naive token-level instantiation --- querying the teacher at every step $t$ and estimating the empirical frequency of the student's exact next token --- is infeasible because the per-token query volume scales as $O(T)$ over reasoning trajectories of length $T$.
Tokenization differences between distinct model families may also cause the student's exact next token to not appear in any teacher rollout, producing sparse, zero-gradient updates that destabilize optimization.
To resolve these issues, \ours instead groups consecutive tokens into chunks of $C$ tokens and applies a single teacher query per chunk, reducing the query volume by a factor of $O(T/C)$ and shifting the matching signal from exact tokens to semantic similarity.
%
% }
% To resolve the severe query overhead of token-level queries, we fundamentally restructure the temporal resolution of distillation. 
%
% Instead of prompting the teacher at every generation step to evaluate individual student tokens, we apply the teacher-critic such that a single teacher query evaluates a contiguous chunk of $C$ tokens, drastically reducing the query volume by a factor of $\mathcal{O}(T/C)$.
%
This chunk-level paradigm is structurally inspired by the verification phase of Speculative Decoding~\citep{leviathan2023fast,chen2023accelerating}, where a large target model parallel-verifies a sequence drafted by a smaller model. 
In our framework, we treat each $C$-token block as a single, cohesive verification unit.

Let the student model generate an on-policy reasoning trajectory $y \sim \pi_\theta(\cdot \mid x)$. 
From this trajectory, we select $M$ non-overlapping chunks of length $C$ (the selection mechanism is later detailed in \secref{subsubsec:chunk_selection}).
For a student chunk $c$, the black-box teacher receives the historical prefix $y_{<c}$ and generates $N$ independent Monte Carlo rollouts $\{y_{\text{teacher}}^{(i)}\}_{i=1}^N$. 
Notably, as frontier teacher models and local student models rarely share the same tokenizer, even at chunk level, computing exact probability matches between teacher rollouts and the student's chunk remains infeasible.
% strict token-level probability matching is structurally impossible.
%
To resolve this tokenizer mismatch, \ours evaluates the student's chunk against the teacher rollouts using a continuous semantic similarity function $\phi : (y_c, y_{\text{teacher}}) \mapsto [0, 1]$
% $\phi \in [0, 1]$ 
(e.g., ROUGE-1 unigram overlap~\citep{lin-2004-rouge}). 
The aggregate semantic similarity $k_{\text{sem}}^{(c)}$ for chunk $c$ is then computed as:
\begin{equation}
    k_{\text{sem}}^{(c)} = \sum_{i=1}^{N} \phi(y_c, y_{\text{teacher}}^{(i)})
    \label{eq:aggregate_semantic_similarity}
\end{equation}
By shifting the measurement from discrete tokens to word-level semantics over a continuous $C$-token window, we effectively verify the student's conceptual milestones across distinct model architectures without penalizing stylistic deviations or vocabulary mismatches.

\subsubsection{Bayesian Smoothing for Sparse Signals}
\label{subsubsec:bayesian_smoothing}
While the chunk-level formulation reduces the query overhead, estimating the teacher policy from a limited sampling budget $N$ introduces high variance. 
A naive frequentist estimation $\hat{\pi}_{\text{freq}}^{(c)} = k_{\text{sem}}^{(c)}/N$ risks gradient explosion during optimization: when the teacher's sparse rollouts yield zero semantic matches, the log-loss is driven to infinity. 

To stabilize this estimation, we model the teacher's feedback using a Dirichlet-Multinomial conjugate prior. 
%
% We first compute the student's normalized probability over the chunk to serve as a structural base prior:
Taking the student's own per-chunk normalized probability as the prior:
\begin{equation}
    \bar{\pi}_\theta^{(c)} = \left( \prod_{t \in c} \pi_\theta(y_t \mid x, y_{<t}) \right)^{\frac{1}{C}}
    \label{eq:student_prior}
\end{equation}
The stabilized Bayesian target proxy $\hat{\pi}_{\text{teacher}}^{(c)}$, parameterized by prior strength $\alpha$ is then:
\begin{equation}
    \hat{\pi}_{\text{teacher}}^{(c)} = \frac{k_{\text{sem}}^{(c)} + \alpha \cdot \bar{\pi}_\theta^{(c)}}{N + \alpha}
\label{eq:bayesian_smooth_estimation}
\end{equation}
%
% Theorem~\ref{theorem:bayesian_bound} (\secref{sec:theoretical_analysis}) shows that this formulation strictly bounds the Mean Squared Error (MSE) of the proxy $\hat{\pi}_{\text{teacher}}^{(c)}$ and guarantees finite, stable gradients even under complete teacher disagreement ($k_{\text{sem}}^{(c)} = 0$), preventing catastrophic divergence during on-policy updates.
Theorem~\ref{thm:gradient_stability} demonstrates that this Bayesian proxy not only strictly bounds the Mean Squared Error (MSE) of $\hat{\pi}_{\text{teacher}}^{(c)}$, but also acts as a mathematical failsafe. By anchoring to the student's prior, it guarantees bounded, non-zero gradients even when teacher rollouts yield zero semantic matches ($k_{\text{sem}}^{(c)} = 0$), completely eliminating the risk of supervision collapse during training.

\subsubsection{Entropy-Driven Chunk Selection
}
\label{subsubsec:chunk_selection}
Under a budget of $M$ chunks per trajectory, the method by which we select these windows determines distillation efficiency. 
Reasoning trajectories oftentimes contain long segments of low-complexity transitions (e.g., formatting and deterministic arithmetic), where teacher verification provides negligible gradient signal. 
%
% Allocating our limited evaluations to these low-complexity segments wastes distillation bandwidth and provides a negligible gradient signal.
%
To maximize the marginal utility of each critic, we aim to target teacher verification at at the high-uncertainty reasoning forks where the student's policy diverges. 

As the student model $\pi_\theta$ is fully accessible, we use its internal predictive entropy as a proxy for this local uncertainty.
At each generation step $t$, we compute the entropy of the student's output vocabulary distribution:
\begin{equation}
    \mathcal{H}_t = - \sum_{v \in \mathcal{V}} \pi_\theta(v \mid x, y_{<t}) \log \pi_\theta(v \mid x, y_{<t})
\end{equation}
where $\mathcal{V}$ is the vocabulary space. 
A low entropy $\mathcal{H}_t$ indicates a deterministic transition; while high entropy indicates a complex decision boundary.

To capture these decision boundaries, we select a set $\mathcal{A}$, which consists of the $M$ tokens with the highest entropy from the trajectory, for chunk formation:
% employ a peak-anchored selection strategy, where we sort the trajectory to identify the set of $M$ anchor tokens, $A$, that exhibit the highest absolute entropy:
\begin{equation}
    \mathcal{A} = \mathop{\arg\max}\limits_{t \subset \{1 \dots T\}, |\mathcal{A}|=M} (\mathcal{H}_t)
\end{equation}
For each anchor token $t^* \in \mathcal{A}$, we extract a contiguous $C$-token chunk that encompasses this point of divergence as the audited chunk.
To maintain structural integrity, any dynamically formed chunks that overlap are merged or resampled.
By anchoring the $M$ chunks to the highest-entropy positions in the trajectory, this selection mechanism ensures that the Bayesian proxy targets ($\hat{\pi}_{\text{teacher}}^{(c)}$) are applied where the student's policy is most uncertain and requires the most correction.

\subsubsection{Trust Region Anchoring and the Total Objective}
\label{subsubsec:trust_region}
While our peak-entropy selection places the $M$ audited chunks at the most informative positions, it leaves the remaining $|\mathcal{U}| = T - M \cdot C$ tokens un-audited. 
%
% Let $U$ denote this set of un-audited tokens, where $|\mathcal{U}| = T - M \cdot C$. 
%
Optimizing a policy on these sparse, localized signals risks policy degeneration, as the student model $\pi_\theta$ might exploit structural shortcuts, generate incoherent text between audited chunks, or collapse its entropy to artificially shorten the trajectory length.

For all un-audited tokens $t \in \mathcal{U}$, we apply a KL divergence penalty against the reference policy $\pi_{\text{ref}}$, which is the frozen copy of the student's initial weights. 
By Pinsker's inequality~\citep{pinsker1964information}, minimizing $D_{KL}(\pi_{\text{ref}} \parallel \pi_\theta)$ bounds the Total Variation shift of the student's policy, acting as a trust region that constrains drift in the unsupervised regions (formalized in \theoremref{theorem:trust_region}).
Consequently, policy updates concentrate on the audited chunks while the trust region holds the un-audited tokens close to $\pi_{\text{ref}}$.
%
% Minimizing the divergence $D_{KL}(\pi_{ref} \parallel \pi_\theta)$ effectively bounds the Total Variation shift of the student's policy. 
%
% This elastic penalty acts as a robust trust region, forcing the student to maintain its foundational fluency and verbose reasoning structure across the unsupervised gaps.
%
% Consequently, the model only allows significant policy updates exactly where the Bayesian teacher explicitly provides corrective gradients.
%

\paragraph{The \ours optimization objective.}
% The final \ours framework seamlessly unifies the Bayesian-smoothed chunk distillation with the token-level trust region penalty. 
Combining chunk-level Bayesian distillation with the token-level trust region, the complete \ours objective is:
%
% Let $\beta$ denote the regularization coefficient controlling the strength of the anchor. 
%
% The complete optimization objective is formulated as:
% \begin{equation}
% \small
%     \mathcal{L}_{\ours}(\theta) = - \mathbb{E}_{\hat{y} \sim \pi_\theta} \left[ \sum_{c=1}^{M} \left( \hat{\pi}_{\text{teacher}}^{(c)} \sum_{t \in c} \log \pi_\theta(y_t \mid x, y_{<t}) \right) \right] + \beta \sum_{t \in U} D_{KL} \Big(\pi_{\text{ref}}(\cdot \mid x, y_{<t}) \parallel \pi_\theta(\cdot \mid x, y_{<t})\Big)
% \label{eq:omniopd_loss}
% \end{equation}
% \zz{I added names to terms so that the theorem section can easier refer to
\begin{equation}
    \mathcal{L}_{\ours}(\theta) = \underbrace{-\mathbb{E}_{\hat{y} \sim \pi_\theta} \!\left[ \sum_{c=1}^{M} \!\left( \hat{\pi}^{(c)}_{\text{teacher}} \sum_{t \in c} \log \pi_\theta(y_t \mid x, y_{<t}) \right) \right]}_{\mathcal{L}_{\text{chunk}}(\theta)} \;+\; \underbrace{\beta \sum_{t \in \mathcal{U}} D_{\text{KL}}\bigl(\pi_{\text{ref}}(\cdot \mid x, y_{<t}) \,\|\, \pi_\theta(\cdot \mid x, y_{<t})\bigr)}_{\text{trust-region anchor}}.
    \label{eq:omniopd_loss}
\end{equation}
% }
%
where $\beta$ controls the strength of the trust-region anchor.
By minimizing $\mathcal{L}_{\ours}$, the student model systematically aligns its critical, high-entropy reasoning junctions with the frontier black-box teacher, while the trust-region penalty ensures the global stability and coherence of the resulting long-horizon trajectory.

\section{Theoretical Analysis of \ours}
\label{sec:theoretical_analysis}
% \zz{
In this section, we provide a theoretical analysis of \ours's three core design components --- chunk-level loss formulation (\eqnref{eq:omniopd_loss}), Bayesian-smoothed estimation (\eqnref{eq:bayesian_smooth_estimation}), and trust-region anchoring on un-audited tokens --- by tying each one to a specific failure mode it prevents.
\secref{subsec:invariant_supervision} then compares the chunk-level supervision in \ours to token-level alternatives, showing that the chunk-level loss is invariant to tokenizer and stylistic variation in teacher rollouts while token-level losses are not.

\paragraph{Setup.}
We assume the semantic similarity function $\phi$ (as in \eqnref{eq:aggregate_semantic_similarity}) takes values in $[0,1]$, teacher rollouts $\{y^{(i)}_{\text{teacher}}\}_{i=1}^N$ are i.i.d.\ samples from $\pi^*_{\text{teacher}}(\cdot \mid x, y_{<c})$, the student prior $\bar{\pi}^{(c)}_\theta > 0$ (as in \eqnref{eq:student_prior}) as a geometric mean of strictly positive per-token probabilities, and the score function $\nabla_\theta \log \pi_\theta(y_t \mid \cdot)$ is bounded, which is a standard assumption for softmax-output transformer policies. 
For each chunk $c$, we define the true expected similarity as
\begin{equation}
    \mu_c := \mathbb{E}_{y' \sim \pi^*_{\text{teacher}}(\cdot \mid x, y_{<c})}\!\left[\phi(y_c, y')\right] \in [0,1].
    \label{eq:mu_c_def}
\end{equation}
The empirical frequentist estimator $\hat{\pi}^{(c)}_{\text{freq}} := k^{(c)}_{\text{sem}} / N$ is the unbiased sample-mean estimator of $\mu_c$ with no prior smoothing. 
The Bayesian proxy in \eqnref{eq:bayesian_smooth_estimation} admits the convex form
\begin{equation}
    \hat{\pi}^{(c)}_{\text{teacher}} = \frac{N}{N+\alpha}\,\hat{\pi}^{(c)}_{\text{freq}} + \frac{\alpha}{N+\alpha}\,\bar{\pi}^{(c)}_\theta.
    \label{eq:bayesian_convex_form}
\end{equation}
All expectations are conditional on the student trajectory $\hat{y}$; only teacher rollouts are random.

\subsection{Gradient Stability of \ours}
\label{subsec:gradient_stability}
Both standard OPD and a hypothetical \ours with empirical frequentist estimator produce gradient signals unstable enough to derail training.
Standard OPD's reverse-KL objective explodes when teacher mass vanishes on student-sampled tokens~\citep{fu2026revisiting}; \ours with empirical estimation collapses when teacher rollouts fail to match the student's chunk. 
Both failures concentrate on the same place where the teacher disagrees with the student, which are precisely the places that carry the most useful learning signal. 
A correct design must close both failure modes simultaneously. 
The theorem in this subsection establishes that \ours does so through two distinct mechanisms operating on two distinct components of the design.

Standard OPD's reverse-KL objective in \eqnref{eq:opd_loss} contains per-token terms of the form $\log\bigl(\pi_\theta(v) / \pi^*_{\text{teacher}}(v)\bigr)$, which diverge whenever the teacher assigns near-zero probability mass to a token the student samples. 
Recent analyses~\citep{fu2026revisiting,li2026rethinking,luo2026demystifying} document this as a real failure mode of token-level OPD. 
The chunk-level loss in \eqnref{eq:omniopd_loss} was constructed to rule out this behavior by design. 
The teacher estimator enters as a bounded multiplier in $[0,1]$ on the student's log-likelihood, not as a denominator or inside a log, so the per-chunk gradient is controlled by the student's score function alone---independent of what the estimator returns or how badly the teacher disagrees with $y_c$.

This loss change does not, however, prevent every gradient failure. 
Under the most direct logit-free estimator---the empirical sample mean $\hat{\pi}^{(c)}_{\text{freq}} = k^{(c)}_{\text{sem}}/N$---the estimator collapses to \emph{exactly} zero whenever none of the $N$ teacher rollouts achieves semantic match with the student chunk. 
A zero multiplier in \eqnref{eq:omniopd_loss} zeros out the per-chunk gradient entirely, leaving the student with no learning signal at exactly the chunks where correction is most valuable. 
The Bayesian prior in \eqnref{eq:bayesian_smooth_estimation} closes this gap. 
Because $\bar{\pi}^{(c)}_\theta > 0$ contributes a strictly positive term to the estimator at every chunk, the multiplier stays bounded away from zero regardless of the rollout outcome.
Essentially, the student's base-model confidence acts as a fallback signal when the teacher disagrees on every rollout.

The Bayesian prior does more than prevent collapse. 
The convex combination $\hat{\pi}^{(c)}_{\text{teacher}} = \tfrac{N}{N+\alpha}\hat{\pi}^{(c)}_{\text{freq}} + \tfrac{\alpha}{N+\alpha}\bar{\pi}^{(c)}_\theta$ introduces a bias-variance tradeoff parameterized by the prior strength $\alpha$: larger $\alpha$ reduces estimator noise but pulls the estimate further from the true teacher, while smaller $\alpha$ inherits the empirical frequentist estimator's variance under small $N$. 
Neither extreme is optimal---the empirical frequentist limit ($\alpha \to 0$) is dominated by sampling noise, while the pure base-prior limit ($\alpha \to \infty$) ignores teacher rollouts entirely. 
Since $\hat{\pi}^{(c)}_{\text{teacher}}$ enters \eqnref{eq:omniopd_loss} as a multiplier on the per-chunk gradient, the estimator's variance carries directly into the gradient's variance, which makes this tradeoff a matter of gradient stability rather than statistical accuracy alone.

\begin{theorem}[Gradient stability of \ours]
\label{thm:gradient_stability}
Under the chunk-level loss in \eqnref{eq:omniopd_loss} and the Bayesian path estimator from \eqnref{eq:bayesian_smooth_estimation}, the following three properties hold.

\textbf{(a) Loss design eliminates the gradient explosion of standard OPD.} 
For every realization of teacher rollouts and every estimator value $\hat{\pi}^{(c)}_{\text{teacher}} \in [0,1]$, the per-chunk gradient is bounded by the student's score function:
\begin{equation}
    \Bigl\| \hat{\pi}^{(c)}_{\text{teacher}} \sum_{t \in c} \nabla_\theta \log \pi_\theta(y_t \mid x, y_{<t}) \Bigr\| \;\leq\; \sum_{t \in c}\bigl\| \nabla_\theta \log \pi_\theta(y_t \mid x, y_{<t}) \bigr\|.
    \label{eq:bounded_chunk_gradient}
\end{equation}
In contrast, the per-token gradient under the reverse-KL objective in \eqnref{eq:opd_loss} is unbounded as $\pi^*_{\text{teacher}}(v \mid x, y_{<t}) \to 0$ for any sampled token $v$.

\textbf{(b) Bayesian prior eliminates the supervision collapse of the empirical estimator.} 
For every realization, including the worst case $k^{(c)}_{\text{sem}} = 0$,
\begin{equation}
    \hat{\pi}^{(c)}_{\text{teacher}} \;\geq\; \frac{\alpha \, \bar{\pi}^{(c)}_\theta}{N + \alpha} \;>\; 0,
    \label{eq:lower_bound_estimator}
\end{equation}
so the per-chunk gradient remains strictly nonzero. 
The empirical frequentist estimator $\hat{\pi}^{(c)}_{\text{freq}}$ does not satisfy this property: under $k^{(c)}_{\text{sem}} = 0$, the chunk's gradient contribution vanishes entirely.

\textbf{(c) Bias-variance characterization.} 
The Bayesian estimator's mean squared error admits the closed-form bias-variance decomposition
\begin{equation}
    \mathrm{MSE}\bigl(\hat{\pi}^{(c)}_{\text{teacher}}\bigr) = \underbrace{\frac{N \sigma^2_\phi}{(N+\alpha)^2}}_{\text{variance}} + \underbrace{\frac{\alpha^2 \bigl(\bar{\pi}^{(c)}_\theta - \mu_c\bigr)^2}{(N+\alpha)^2}}_{\text{squared bias}},
    \label{eq:mse_decomposition}
\end{equation}
where $\sigma^2_\phi := \mathrm{Var}_{y' \sim \pi^*_{\text{teacher}}}[\phi(y_c, y')] \leq 1/4$ by Popoviciu's inequality~\citep{popoviciu1965certaines}. 
The two terms admit a finite optimum $\alpha^*$, and the variance strictly contracts relative to the empirical frequentist baseline: $\mathrm{Var}\bigl(\hat{\pi}^{(c)}_{\text{teacher}}\bigr) = \bigl(\tfrac{N}{N+\alpha}\bigr)^2 \mathrm{Var}\bigl(\hat{\pi}^{(c)}_{\text{freq}}\bigr) < \mathrm{Var}\bigl(\hat{\pi}^{(c)}_{\text{freq}}\bigr)$.
\end{theorem}

\begin{proof}
\emph{(a)} Since $\hat{\pi}^{(c)}_{\text{teacher}} \in [0,1]$ as a convex combination of $\hat{\pi}^{(c)}_{\text{freq}} \in [0,1]$ and $\bar{\pi}^{(c)}_\theta \in [0,1]$, the triangle inequality together with $\hat{\pi}^{(c)}_{\text{teacher}} \leq 1$ gives \eqnref{eq:bounded_chunk_gradient}. 
For the reverse-KL contrast, the gradient
\begin{equation*}
\nabla_\theta D_{\mathrm{KL}}\bigl(\pi_\theta \,\|\, \pi^*_{\text{teacher}}\bigr) = \sum_v \nabla_\theta \pi_\theta(v) \,\log\frac{\pi_\theta(v)}{\pi^*_{\text{teacher}}(v)}
\end{equation*}
contains terms whose magnitude diverges as $\pi^*_{\text{teacher}}(v) \to 0$ for any token $v$ with $\pi_\theta(v) > 0$.

\emph{(b)} From the convex form $\hat{\pi}^{(c)}_{\text{teacher}} = (k^{(c)}_{\text{sem}} + \alpha \bar{\pi}^{(c)}_\theta)/(N+\alpha)$, the setup assumption $\bar{\pi}^{(c)}_\theta > 0$, and the non-negativity of the semantic match count $k^{(c)}_{\text{sem}} \geq 0$,
\begin{equation*}
    \hat{\pi}^{(c)}_{\text{teacher}} \;=\; \frac{k^{(c)}_{\text{sem}} + \alpha \bar{\pi}^{(c)}_\theta}{N+\alpha} \;\geq\; \frac{\alpha \bar{\pi}^{(c)}_\theta}{N+\alpha} \;>\; 0.
\end{equation*}
The empirical frequentist estimator $\hat{\pi}^{(c)}_{\text{freq}} = k^{(c)}_{\text{sem}}/N$ admits no such lower bound: when $k^{(c)}_{\text{sem}} = 0$, $\hat{\pi}^{(c)}_{\text{freq}} = 0$.

\emph{(c)} Conditional on the student trajectory $\hat{y}$, the prior $\bar{\pi}^{(c)}_\theta$ is fixed and only the empirical frequentist estimator is random. 
Since the rollouts are i.i.d.\ with $\phi(y_c, y^{(i)}_{\text{teacher}}) \in [0,1]$,
\begin{equation*}
    \mathbb{E}\bigl[\hat{\pi}^{(c)}_{\text{freq}}\bigr] = \mu_c, \quad \mathrm{Var}\bigl(\hat{\pi}^{(c)}_{\text{freq}}\bigr) = \frac{\sigma^2_\phi}{N}, \quad \sigma^2_\phi \leq \tfrac{1}{4} \text{ by Popoviciu's inequality~\citep{popoviciu1965certaines}.}
\end{equation*}
Applying the convex form from \eqnref{eq:bayesian_convex_form}, the bias and variance of the Bayesian estimator decompose as
\begin{align*}
    \mathrm{Bias}\bigl(\hat{\pi}^{(c)}_{\text{teacher}}\bigr) &= \frac{N}{N+\alpha}\mu_c + \frac{\alpha}{N+\alpha}\bar{\pi}^{(c)}_\theta - \mu_c \;=\; \frac{\alpha\bigl(\bar{\pi}^{(c)}_\theta - \mu_c\bigr)}{N+\alpha}, \\
    \mathrm{Var}\bigl(\hat{\pi}^{(c)}_{\text{teacher}}\bigr) &= \Bigl(\frac{N}{N+\alpha}\Bigr)^2 \mathrm{Var}\bigl(\hat{\pi}^{(c)}_{\text{freq}}\bigr) \;=\; \frac{N \sigma^2_\phi}{(N+\alpha)^2}.
\end{align*}
Summing squared bias and variance yields \eqnref{eq:mse_decomposition}. The variance contraction is immediate: $(N/(N+\alpha))^2 < 1$ for all $\alpha > 0$.
\end{proof}

\subsection{Rollout Sufficiency of \ours}
\label{subsec:rollout_sufficiency}
\ours's efficiency advantage over standard OPD rests entirely on the rollout budget per chunk being small. 
The per-sample cost analysis in \appref{app:cost_analysis} shows that the default configuration ($M=10$, $N=10$, $C=50$) runs at 1.75$\times$ SFT cost and the sparser variant ($M=5$, $N=10$, $C=50$) runs at 0.88$\times$, but this favorable cost profile collapses entirely if accurate estimation requires $N$ on the order of hundreds or thousands. 
The theorem in this subsection establishes that $N=10$ is principled rather than budget-constrained, by characterizing how the estimator's accuracy improves with the rollout budget $N$ and predicting the diminishing returns observed in \tabref{tab:cost_analysis}.

The remaining question is whether a sparse configuration such as $N = 10$ gives a usable signal. 
Two effects work in \ours's favor. 
First, the empirical estimator $\hat{\pi}^{(c)}_{\text{freq}}$ is a sample mean over $N$ i.i.d.\ rollouts bounded in $[0,1]$, so by Hoeffding's inequality~\citep{hoeffding1963probability} its deviation from $\mu_c$ shrinks at sub-Gaussian rate in $N$. 
Second, the Bayesian shrinkage in \eqnref{eq:bayesian_smooth_estimation} tightens this further by a factor of $N/(N+\alpha) < 1$. 
Since $\hat{\pi}^{(c)}_{\text{teacher}}$ enters \eqnref{eq:omniopd_loss} as a multiplier on the per-chunk gradient, both the estimator's accuracy and the gradient noise improve at rate $\mathcal{O}(1/N)$. 
The marginal benefit of doubling $N$ therefore halves with each doubling, predicting a sweet spot beyond which additional rollouts purchase variance reduction that the optimizer cannot exploit.

\begin{theorem}[Concentration of the chunk-level estimator]
\label{thm:concentration}
For any $\varepsilon > 0$, the Bayesian chunk-level estimator concentrates around its shrunken target $\tilde{\mu}_c := \tfrac{N}{N+\alpha}\mu_c + \tfrac{\alpha}{N+\alpha}\bar{\pi}^{(c)}_\theta$ with sub-Gaussian rate
\begin{equation}
    \Pr\!\left(\big|\hat{\pi}^{(c)}_{\text{teacher}} - \tilde{\mu}_c\big| \geq \varepsilon\right) \leq 2\exp\!\left(-\frac{2(N+\alpha)^2 \varepsilon^2}{N}\right).
    \label{eq:concentration_bound}
\end{equation}
Consequently, the estimator-induced contribution to the per-chunk gradient variance is bounded by
\begin{equation}
    \mathrm{Var}\!\left(\hat{\pi}^{(c)}_{\text{teacher}}\sum_{t \in c} \nabla_\theta \log \pi_\theta(y_t \mid x, y_{<t}) \,\Big|\, \hat{y}\right) \leq \frac{N}{4(N+\alpha)^2}\Big\|\sum_{t \in c}\nabla_\theta \log \pi_\theta(y_t \mid x, y_{<t})\Big\|^2,
    \label{eq:concentration_variance}
\end{equation}
shrinking at rate $\mathcal{O}(1/N)$ in the rollout budget.
\end{theorem}

\begin{proof}
Conditional on the student trajectory $\hat{y}$, decompose 
$\hat{\pi}^{(c)}_{\text{teacher}} - \tilde{\mu}_c = \tfrac{N}{N+\alpha}\bigl(\hat{\pi}^{(c)}_{\text{freq}} - \mu_c\bigr)$. 
Applying Hoeffding's inequality~\citep{hoeffding1963probability} to $\hat{\pi}^{(c)}_{\text{freq}}$ as an average of $N$ i.i.d.\ samples in $[0,1]$,
\begin{equation*}
    \Pr\!\left(\big|\hat{\pi}^{(c)}_{\text{freq}} - \mu_c\big| \geq \delta\right) \leq 2\exp(-2N\delta^2),
\end{equation*}
and substituting $\delta = \tfrac{N+\alpha}{N}\varepsilon$ yields \eqnref{eq:concentration_bound}.
For the variance bound, the score sum $\sum_{t \in c} \nabla_\theta \log \pi_\theta(y_t \mid x, y_{<t})$ is deterministic conditional on $\hat{y}$, so the gradient variance factorizes as $\mathrm{Var}\bigl(\hat{\pi}^{(c)}_{\text{teacher}}\bigr) \cdot \bigl\|\sum_{t \in c} \nabla_\theta \log \pi_\theta(y_t \mid x, y_{<t})\bigr\|^2$. 
Substituting $\mathrm{Var}\bigl(\hat{\pi}^{(c)}_{\text{teacher}}\bigr) \leq N / 4(N+\alpha)^2$ from \theoremref{thm:gradient_stability}(c) gives \eqnref{eq:concentration_variance}.
\end{proof}

\subsection{Un-audited Token Stability of \ours}
\label{subsec:un_audited_stability}
\ours audits only $M \cdot C$ tokens per trajectory and leaves the remaining $T - M \cdot C$ tokens un-audited. 
This sparsity is the design feature that makes \ours efficient---one teacher query per chunk rather than per token. 
However, without any constraint on the un-audited region, the student can exploit the sparse audited supervision in ways that degrade the trajectory globally, including collapsing entropy on un-audited tokens to shorten trajectories and reduce penalty mass, generating incoherent connective text between audited chunks, or shifting probability mass on un-audited tokens to satisfy gradients from the audited chunks at the cost of long-horizon fluency.

The KL anchor against the frozen base policy $\pi_{\text{ref}}$ closes these failure modes by constraining the un-audited region directly. 
At each un-audited token, Pinsker's inequality~\citep{pinsker1964information} bounds the total-variation distance between $\pi_{\text{ref}}$ and $\pi_\theta$ by the per-token KL, so reducing per-token KL through the anchor term directly bounds the per-token drift. 
At the aggregate level, $\theta_{\text{ref}}$ is always a feasible point of the objective at which the anchor's per-token KL terms vanish, so the optimality of $\theta^*$ forces an aggregate constraint. 
The total KL summed over un-audited tokens at $\theta^*$ cannot exceed the chunk-level loss improvement $\mathcal{L}_{\text{chunk}}(\theta_{\text{ref}}) - \mathcal{L}_{\text{chunk}}(\theta^*)$ divided by $\beta$.
Beyond a regularization weight, $\beta$ therefore has a concrete operational meaning, i.e., $\beta^{-1}$ is the exchange rate at which \ours converts chunk-level loss improvement into permitted un-audited drift.
\theoremref{theorem:trust_region} below makes both bounds precise---the pointwise total-variation control from Pinsker's inequality, and the aggregate KL budget that scales inversely with $\beta$.

\begin{theorem}[Trust-region bound on un-audited tokens]
\label{theorem:trust_region}
At any iterate $\theta$ and for every un-audited position $t \in \mathcal{U}$,
\begin{equation}
    D_{\mathrm{TV}}\!\bigl(\pi_{\text{ref}}(\cdot \mid x, y_{<t}) \,\|\, \pi_\theta(\cdot \mid x, y_{<t})\bigr) \leq \sqrt{\tfrac{1}{2}D_{\mathrm{KL}}\!\bigl(\pi_{\text{ref}}(\cdot \mid x, y_{<t}) \,\|\, \pi_\theta(\cdot \mid x, y_{<t})\bigr)}.
    \label{eq:tv_pinsker}
\end{equation}
Furthermore, any minimizer $\theta^* = \arg\min_\theta \mathcal{L}_{\ours}(\theta)$ of \eqnref{eq:omniopd_loss} satisfies the aggregate budget
\begin{equation}
    \sum_{t \in \mathcal{U}} D_{\mathrm{KL}}\!\bigl(\pi_{\text{ref}}(\cdot \mid x, y_{<t}) \,\|\, \pi_{\theta^*}(\cdot \mid x, y_{<t})\bigr) \;\leq\; \frac{1}{\beta}\Bigl[\mathcal{L}_{\text{chunk}}(\theta_{\text{ref}}) - \mathcal{L}_{\text{chunk}}(\theta^*)\Bigr].
    \label{eq:kl_budget}
\end{equation}
As $\beta \to \infty$, the aggregate un-audited drift vanishes and $\pi_{\theta^*} \to \pi_{\text{ref}}$ on $\mathcal{U}$.
\end{theorem}

\begin{proof}
The pointwise total-variation bound in \eqnref{eq:tv_pinsker} is Pinsker's inequality applied at each $t \in \mathcal{U}$. 
For the aggregate budget in \eqnref{eq:kl_budget}, evaluate $\mathcal{L}_{\ours}$ at the feasible point $\theta_{\text{ref}}$, where the KL anchor term vanishes by construction:
\begin{equation*}
    \mathcal{L}_{\text{chunk}}(\theta^*) \;+\; \beta \sum_{t \in \mathcal{U}} D_{\mathrm{KL}}\bigl(\pi_{\text{ref}} \,\|\, \pi_{\theta^*}\bigr) \;\leq\; \mathcal{L}_{\ours}(\theta^*) \;\leq\; \mathcal{L}_{\ours}(\theta_{\text{ref}}) \;=\; \mathcal{L}_{\text{chunk}}(\theta_{\text{ref}}),
\end{equation*}
using optimality of $\theta^*$ and the vanishing of the KL term at $\theta_{\text{ref}}$. 
Rearranging yields the stated budget. 
The limiting statement $\pi_{\theta^*} \to \pi_{\text{ref}}$ as $\beta \to \infty$ follows because the right-hand side is bounded above by a constant independent of $\beta$, so the aggregate KL must tend to zero.
\end{proof}

\subsection{Tokenizer- and Style-Invariant Supervision}
\label{subsec:invariant_supervision}
In the logit-free setting that \ours addresses, the natural token-level alternative is sequence-level distillation~\citep{kim2016sequence}---sample $N$ teacher rollouts and apply token-level cross-entropy on each rollout as if it were a ground-truth sequence. 
The resulting supervision forces the student to match the teacher's surface-form choices token by token, conflating semantic content with stylistic and formatting decisions. 
Standard OPD, although not directly comparable in the logit-free setting, exhibits a structurally analogous property at the distributional level---its reverse-KL objective in \eqnref{eq:opd_loss} contains terms $\log \pi^*_{\text{teacher}}(v \mid x, y_{<t})$ summed against every vocabulary token $v$, so the student inherits the teacher's full per-token preferences, including distinctions that are semantically irrelevant.

\ours's chunk-level loss decouples semantic content from surface form by routing all teacher information through the semantic similarity function $\phi$. 
The Bayesian estimator $\hat{\pi}^{(c)}_{\text{teacher}}$ depends on rollouts only through $k^{(c)}_{\text{sem}}$, and the loss in \eqnref{eq:omniopd_loss} depends on rollouts only through this estimator. 
Any invariance property $\phi$ has is therefore inherited by the chunk-level loss. 
\theoremref{thm:invariant_supervision} below formalizes this inheritance and contrasts it with the two token-level alternatives above.

\begin{theorem}[Tokenizer and stylistic invariance of the chunk-level loss]
\label{thm:invariant_supervision}
Let $\sim_\phi$ denote the equivalence relation on token sequences induced by $\phi$: $y \sim_\phi y'$ iff $\phi(z, y) = \phi(z, y')$ for every $z$. 
For any two sets of teacher rollouts $\{y^{(i)}_{\text{teacher}}\}_{i=1}^N$ and $\{y'^{(i)}_{\text{teacher}}\}_{i=1}^N$ that are $\sim_\phi$-equivalent (i.e., $y^{(i)}_{\text{teacher}} \sim_\phi y'^{(i)}_{\text{teacher}}$ for every $i$), the chunk-level Bayesian estimator and the chunk-level loss in \eqnref{eq:omniopd_loss} are invariant:
\begin{equation}
    \hat{\pi}^{(c)}_{\text{teacher}}\bigl(\{y^{(i)}_{\text{teacher}}\}\bigr) = \hat{\pi}^{(c)}_{\text{teacher}}\bigl(\{y'^{(i)}_{\text{teacher}}\}\bigr), \qquad \mathcal{L}_{\text{chunk}}\bigl(\theta; \{y^{(i)}_{\text{teacher}}\}\bigr) = \mathcal{L}_{\text{chunk}}\bigl(\theta; \{y'^{(i)}_{\text{teacher}}\}\bigr).
    \label{eq:invariant_supervision}
\end{equation}
In contrast, the sequence-level distillation loss
\begin{equation}
    \mathcal{L}_{\text{SFT}}\bigl(\theta; \{y^{(i)}_{\text{teacher}}\}\bigr) := -\frac{1}{N} \sum_{i=1}^N \sum_{t} \log \pi_\theta\bigl(y^{(i)}_{\text{teacher}, t} \mid x, y^{(i)}_{\text{teacher}, <t}\bigr)
    \label{eq:sft_rollouts_loss}
\end{equation}
depends on specific teacher tokens and is not invariant under $\sim_\phi$-substitution. 
The standard OPD loss in \eqnref{eq:opd_loss} depends on the teacher's per-token distribution $\pi^*_{\text{teacher}}$ and distinguishes $\sim_\phi$-equivalent tokens whenever the teacher assigns them different probabilities.
\end{theorem}

\begin{proof}
The chunk-level loss in \eqnref{eq:omniopd_loss} depends on the teacher rollouts $\{y^{(i)}_{\text{teacher}}\}$ only through the Bayesian estimator $\hat{\pi}^{(c)}_{\text{teacher}}$ from \eqnref{eq:bayesian_smooth_estimation}, which in turn depends on the rollouts only through $k^{(c)}_{\text{sem}} = \sum_i \phi(y_c, y^{(i)}_{\text{teacher}, c})$. 
By the definition of $\sim_\phi$, $\phi(y_c, y^{(i)}_{\text{teacher}, c}) = \phi(y_c, y'^{(i)}_{\text{teacher}, c})$ for every $i$, so $k^{(c)}_{\text{sem}}$ is invariant, $\hat{\pi}^{(c)}_{\text{teacher}}$ is invariant, and $\mathcal{L}_{\text{chunk}}$ is invariant. 
The non-invariance of the sequence-level distillation loss in \eqnref{eq:sft_rollouts_loss} is immediate from its definition---it contains $\log \pi_\theta(y^{(i)}_{\text{teacher}, t} \mid \cdot)$ evaluated at the specific teacher tokens, which generally differ between $\sim_\phi$-equivalent rollouts. 
For standard OPD, \eqnref{eq:opd_loss} sums $\pi_\theta(v \mid \cdot) \log \pi^*_{\text{teacher}}(v \mid \cdot)$ over the entire vocabulary, so the loss is sensitive to any per-token distinction in $\pi^*_{\text{teacher}}$, including those between $\sim_\phi$-equivalent tokens.
\end{proof}

\begin{remark}[Three invariance classes]
\label{rem:invariance_classes}
The choice of $\phi$ determines which surface-level variations the chunk-level loss treats as equivalent, yielding three practically important invariance properties.

\textbf{Tokenizer invariance.} 
If $\phi$ operates on detokenized text---for instance, an embedding-based semantic similarity over decoded strings---then the chunk loss is invariant to the teacher's tokenizer. 
Teachers with different vocabularies and tokenization schemes yield the same supervision signal up to $\phi$.

\textbf{Paraphrase invariance.} 
If $\phi$ captures chunk-level semantic equivalence, teacher rollouts that paraphrase the same content yield identical chunk loss. 
The student is not penalized for choosing a different surface realization of the same semantic content.

\textbf{Stylistic invariance.} 
If $\phi$ is insensitive to register, formatting, or hedging conventions, teacher rollouts that differ only in stylistic dimensions yield identical chunk loss.

The empirical generalization of \ours across the Qwen3-32B~\citep{yang2025qwen3} and Gemini-2.5-flash~\citep{comanici2025gemini} teachers in later \secref{sec:result_analysis}---two model families with different tokenizers, training data, and stylistic preferences---is consistent with these invariance properties.
\end{remark}
\section{Experiments}
\label{sec:experiments}

\subsection{Experimental Setup}
\label{subsec:setup}

\paragraph{Models.}
We use Qwen3-1.7B and Qwen3-4B~\citep{yang2025qwen3} as the student models. 
To demonstrate the wide adaptivity of \ours, we let teacher models span both open-weight and proprietary frontier models, including Qwen3-32B, Qwen3-30B-A3B-Instruct~\citep{yang2025qwen3} on the open-weight side; and Gemini-2.5-Flash~\citep{huang2025gemini}, Claude-4.5-Haiku~\citep{Claude-4} on the proprietary side.

\paragraph{Datasets and evaluation metrics.} 
For mathematical reasoning, we train on DAPO-Math-17K~\citep{yu2025dapo} and evaluate on AIME-2024~\citep{li2024numinamath}, AIME-2025~\citep{dekoninck2026matharena}, AMC23~\citep{li2024numinamath}, MATH-500~\citep{hendrycks2021measuring}, and OlympiadBench~\citep{he2024olympiadbench}. 
For competitive programming, following~\citet{liang2025beyond}, we train on the PRIME-RL compilation~\citep{cui2025process}, which includes APPS~\citep{hendrycksapps2021}, CodeContests~\citep{codecontests}, TACO~\citep{li2023taco}, and Codeforces~\citep{penedo2025codeforces}, and evaluate on 100 randomly sampled validation instances per dataset category.
To ensure statistical robustness, we report Pass@1 averaged over 16 independent sampling trials.

\paragraph{Baselines.} 
We compare against three baselines, one for each teacher-access setting.
GRPO~\citep{shao2024deepseekmath} serves as the self-exploratory RL baseline that requires no teacher. 
Offline SFT on teacher rollouts~\citep{kim2016sequence} is the standard logit-free distillation baseline, which we apply across all teachers. 
For open-weight teachers, where logits are accessible, we additionally compare against the standard white-box OPD approach~\citep{lu2025onpolicydistillation}.

\paragraph{Implementation Details.}
We train all student models for 1 epoch with learning rate $1 \times 10^{-6}$ and a per-device batch size of 2. 
To accommodate extended reasoning trajectories, the maximum sequence length is 20K tokens during training and extended to 32K tokens during evaluation.
Student sampling temperature is 1.0. 
For \ours, the default configuration is $M=10$ chunks, $C=50$ tokens per chunk, $N=10$ Monte Carlo rollouts per chunk, KL anchor weight $\beta=0.1$, and Edit Distance~\citep{navarro2001guided} as the default semantic similarity metric $\phi$.
For the white-box OPD baseline, we use a balanced objective of 0.5 Forward KL + 0.5 Reverse KL. 
To ensure fair comparison under the same budget, we limit the training process to 50 optimization steps for both the SFT baselines and \ours when distilling from Claude-4.5-Haiku and Gemini-2.5-Flash.

\subsection{Main Results}
\label{sec:result_analysis}
We present the main results for mathematical reasoning and competitive programming in \tabref{tab:main_results_math} and~\ref{tab:main_results_coding}.
\begin{table*}[ht]
\centering
\caption{
    % Main results on mathematical reasoning benchmarks. 
    % %
    % Standard OPD requires white-box access and is therefore inapplicable to black-box models (Claude-4.5-Haiku, Gemini-2.5-Flash), highlighting the utility of \ours. 
    % %
    % Best results among distillation methods for each setup are in \textbf{bold}.
    Main results on mathematical reasoning.
    Pass@1 (\%) performance across AIME-2024, AIME-2025, AMC23, MATH-500, and OlympiadBench.
    Each student (Qwen3-1.7B, Qwen3-4B) is distilled from four teachers spanning open-weight (Qwen3-32B, Qwen3-30B-A3B-Instruct) and proprietary frontier (Claude-4.5-Haiku, Gemini-2.5-Flash) models.
    For each student--teacher pair, we report \textit{base inference} (no training), \textit{GRPO} (self-exploratory RL with rule-based rewards, no teacher), \textit{SFT} (offline imitation of teacher rollouts), and \textit{standard OPD} (requires teacher logits; only applicable to open-weight teachers) as comparisons.
    Best results per student configuration are in \textbf{bold}.
}
\vspace{-0.1in}
\label{tab:main_results_math}
\setlength{\tabcolsep}{4pt}
\resizebox{\textwidth}{!}{%
\begin{tabular}{llccccccc}
\toprule
\textbf{Student Model} & \textbf{Teacher Model} & \textbf{Method} & \textbf{MATH-500} & \textbf{AIME-24} & \textbf{AIME-25} & \textbf{AMC-23} & \textbf{OlympiadBench} & \textbf{Average} \\
\midrule
\multicolumn{9}{c}{\textit{Reference Models (Zero-Shot / Base Inference)}} \\
\midrule
- & Qwen3-32B & Base & 73.60 & 75.42 & 68.33 & 87.81 & 52.93 & 71.62 \\
- & Qwen3-30B-A3B-Inst. & Base & 69.81 & 61.04 & 44.79 & 78.59 & 56.38 & 62.12 \\
- & Claude-4.5-haiku & Base & 89.95 & 55.20 & 34.79 & 87.50 & 62.69 & 66.03 \\
- & Gemini-2.5-flash & Base & 86.83 & 74.79 & 65.20 & 93.13 & 61.86 & 76.36 \\
\midrule
\multicolumn{9}{c}{\textit{Student: Qwen3-1.7B}} \\
\midrule
Qwen3-1.7B & - & Base & 60.26 & 48.54 & 37.50 & 75.16 & 37.90 & 51.87 \\
Qwen3-1.7B & - & GRPO & 80.35 & 47.50 & 44.45 & 84.06 & 55.67 & 62.41 \\
\cmidrule{1-9}
\multirow{3}{*}{Qwen3-1.7B} & \multirow{3}{*}{Qwen3-32B} & SFT & 79.91 & 35.42 & 30.20 & 71.09 & 52.19 & 53.76 \\
 & & OPD & \textbf{84.82} & \textbf{48.95} & \textbf{37.58} & 79.06 & 58.08 & \textbf{61.70} \\
 & & \ours & 84.61 & 47.08 & 33.33 & \textbf{79.68} & \textbf{58.64} & 60.67 \\
\midrule
\multicolumn{9}{c}{\textit{Student: Qwen3-4B}} \\
\midrule
Qwen3-4B & - & Base & 61.55 & 50.83 & 41.88 & 75.47 & 40.32 & 54.01 \\
Qwen3-4B & - & GRPO & 73.75 & 68.62 & 60.70 & 95.93 & 52.19 & 70.24 \\
\cmidrule{1-9}
\multirow{3}{*}{Qwen3-4B} & \multirow{3}{*}{Qwen3-32B} & SFT & 68.46 & 60.83 & 58.95 & 83.59 & 47.19 & 63.80 \\
 & & OPD & 67.50 & 61.45 & 60.25 & 82.65 & 48.96 & 64.16 \\
 & & \ours & \textbf{72.60} & \textbf{68.67} & \textbf{60.42} & \textbf{93.44} & \textbf{50.27} & \textbf{69.08} \\
\cmidrule{1-9}
\multirow{3}{*}{Qwen3-4B} & \multirow{3}{*}{Qwen3-30B-A3B-Inst.} & SFT & 66.77 & 34.58 & 31.04 & 70.16 & 46.29 & 49.77 \\
 & & OPD & 69.78 & 48.54 & 36.87 & 75.78 & 50.14 & 56.22 \\
 & & \ours & \textbf{74.98} & \textbf{75.00} & \textbf{62.92} & \textbf{92.65} & \textbf{56.05} & \textbf{72.32} \\
\cmidrule{1-9}
\multirow{2}{*}{Qwen3-4B} & \multirow{2}{*}{Claude-4.5-haiku} & SFT & 80.56 & 60.83 & 53.33 & 85.16 & 57.71 & 67.52 \\
 & & \ours & \textbf{81.28} & \textbf{72.08} & \textbf{65.63} & \textbf{91.25} & \textbf{64.35} & \textbf{74.92} \\
\cmidrule{1-9}
\multirow{2}{*}{Qwen3-4B} & \multirow{2}{*}{Gemini-2.5-flash} & SFT & 80.10 & 71.25 & 63.13 & 90.44 & 62.61 & 73.51 \\
 & & \ours & \textbf{83.30} & \textbf{73.13} & \textbf{64.37} & \textbf{92.96} & \textbf{64.58} & \textbf{75.67} \\
\bottomrule
\end{tabular}%
}
\end{table*}
\begin{table*}[t]
\centering
\caption{
    Main results on competitive programming.
    Pass@1 (\%) on 100 randomly sampled validation instances per dataset across APPS, CodeContests, TACO, and Codeforces.
    Only the Qwen3-32B teacher is used; students, baselines, and formatting conventions follow \tabref{tab:main_results_math}.
}
\vspace{-0.1in}
\label{tab:main_results_coding}
\setlength{\tabcolsep}{12pt}
\resizebox{\textwidth}{!}{%
\begin{tabular}{llcccccc}
\toprule
\textbf{Student Model} & \textbf{Teacher Model} & \textbf{Method} & \textbf{APPS} & \textbf{CodeContests} & \textbf{Codeforces} & \textbf{TACO} & \textbf{Average} \\
\midrule
\multirow{3}{*}{Qwen3-1.7B} & \multirow{3}{*}{Qwen3-32B} & SFT & 46.78 & 45.16 & 43.39 & 26.43 & 40.44 \\
 & & OPD & 52.67 & 52.09 & \textbf{52.89} & 30.61 & 47.06 \\
 & & \ours & \textbf{54.28} & \textbf{52.94} & 51.55 & \textbf{32.96} & \textbf{47.93} \\
\midrule
\multirow{3}{*}{Qwen3-4B} & \multirow{3}{*}{Qwen3-32B} & SFT & 69.82 & 65.59 & 68.60 & 44.12 & 62.03 \\
 & & OPD & \textbf{72.72} & \textbf{68.86} & \textbf{74.33} & \textbf{45.10} & \textbf{65.26} \\
 & & \ours & 71.72 & 68.52 & 70.49 & 44.37 & 63.78 \\
\bottomrule
\end{tabular}%
}
\end{table*}

\paragraph{Improvements over offline SFT.} 
In both domains, \ours achieves substantial performance gains over the standard SFT across all teacher choices. 
For instance, applying \ours to the Qwen3-4B student yields an average mathematical accuracy of 69.08\% using the open-weight Qwen3-32B as the teacher. 
This represents a +15.07\% absolute improvement over base model inference (54.01\%) and a +5.28\% gain over the offline SFT baseline (63.80\%). 
Crucially, this improvement scales when integrating frontier black-box models. 
When distilling from Claude-4.5-Haiku, \ours achieves 74.92\%, decisively outperforming its corresponding SFT baseline (67.52\%) by +7.40\%. 
Similarly, using Gemini-2.5-Flash as the teacher yields 75.67\%, directly surpassing the Gemini-derived SFT baseline (73.51\%) by +2.16\%. 
These results validate that learning from the student's own actively generated reasoning trajectories, and systematically correcting its specific bottlenecks, provides a fundamentally richer learning signal than merely imitating offline demonstrations.

\paragraph{Outperforming standard white-box OPD.} 
Besides improvements over SFT baselines, \ours outperforms standard OPD~\citep{lu2025onpolicydistillation} as well, despite it having full access to the teacher's output logits. 
In the Qwen3-4B (student) and Qwen3-32B (teacher) setup, \ours achieves an average of 69.08\% compared to OPD's 64.16\%. 
The gap widens sharply when the teacher is more stylistically different from the base student. 
With Qwen3-30B-A3B-Instruct as the teacher (a heavily aligned instruct model), \ours achieves 72.32\% while OPD reaches only 56.22\%.
This empirical result aligns with \theoremref{thm:invariant_supervision}, which shows that standard OPD's reverse-KL objective is sensitive to every per-token distinction in $\pi^*_{\text{teacher}}$, so the student inherits all the teacher's surface-form preferences, including stylistic choices that differ from the student's natural phrasing.
\ours's chunk-level loss is invariant to such surface-form variation --- the loss depends on teacher rollouts only through the semantic similarity $\phi$.
With more stylistically different teachers like Qwen3-30B-A3B-Instruct, the stylistic gap from the base student is larger, OPD's per-token penalties become correspondingly more punishing, and \ours's invariance translates into a wider empirical margin.
On competitive programming (\tabref{tab:main_results_coding}), \ours edges OPD on the Qwen3-1.7B (student)/Qwen3-32B (teacher) setup, but slightly underperforms on the Qwen3-4B (student) setup. 
We hypothesize that code's syntactic rigidity makes token-level matching more directly meaningful than in mathematical reasoning prose, so chunk-level invariance helps less here than on math.
%
% We hypothesize this advantage stems from our chunk-level Bayesian formulation acting as a statistical regularizer. Standard OPD enforces strict token-by-token imitation, heavily penalizing the student for valid alternative phrasing. 
%
% By verifying multi-step chunks via semantic similarity, \ours grants the student local structural flexibility. 
%
% It extracts a smoother, more robust gradient signal that focuses on reaching correct reasoning milestones rather than matching deterministic token sequences.

\paragraph{Teacher capacity scaling and comparisons to RL.} 
The impact of teacher capacity on distillation efficacy becomes particularly evident when comparing our method against RL approaches like GRPO~\citep{shao2024deepseekmath}. 
When paired with a relatively bounded open-weight teacher (e.g., Qwen3-32B teaching Qwen3-1.7B), \ours (60.67\%) slightly trails the self-improving GRPO baseline (62.41\%). 
This suggests that when the teacher's capability margin over the student is small, the self-exploration of GRPO under a rule-based environment reward remains highly competitive. 
However, as the teacher setup scales to frontier models, \ours rapidly unlocks performance ceilings that self-rewarding GRPO cannot reach. 
With Claude-4.5-Haiku or Gemini-2.5-Flash teaching the Qwen3-4B student, \ours reaches 74.92\% and 75.67\% respectively, surpassing the 4B GRPO performance (70.24\%). 
%
% When the Qwen3-4B student is trained via \ours using Claude-4.5-Haiku (74.92\%) or Gemini-2.5-Flash (75.67\%), it outperforms the 4B GRPO baseline (70.24\%). 
%
% \zz{
The consistency of these gains across teacher families --- Qwen3, Claude, and Gemini, which differ in tokenizers, training data, and stylistic conventions --- aligns with \theoremref{thm:invariant_supervision}, which establishes that the chunk-level loss is invariant to such surface-form variation in teacher rollouts. 
This confirms that \ours bridges the open/closed-source distillation divide, letting smaller open-weight students extract teacher signal from frontier proprietary models (previously inaccessible to standard OPD) and reach performance beyond what self-exploratory RL alone can achieve.
% }
% This confirms that \ours effectively bridges the capability divide between open and closed-source models, allowing smaller student networks to harness the vast reasoning capacity of frontier models to surpass standard RLHF limits.
%
\section{Analysis and Ablation Studies}
\subsection{Empirical Verification of the Bayesian Estimator}
\label{subsec:bayesian_estimator_analysis}
% \zz{WIP, needs some new structuring}

% \zz{
In this section, we empirically verify the Bayesian estimator's bias--variance characterization from \theoremref{thm:gradient_stability}.
While the theorem characterizes the trade-off symbolically (the Bayesian estimator's variance contracts by $(N/(N+\alpha))^2$ at the cost of an explicit bias term), the net empirical benefit depends on two problem-specific quantities that the theorem alone does not pin down, i.e., the rollout noise $\sigma^2_\phi$ and the prior-target gap $\bar{\pi}^{(c)}_\theta - \mu_c$.
Since this estimator enters the chunk-level loss as a per-chunk gradient multiplier, its empirical MSE directly affects gradient stability during training. 
We verify this benefit empirically under different realistic settings of $\phi$ and $N$ below.
% }
% Before evaluating downstream reasoning performance, we validate the accuracy of our proxy estimation method. 
%

To establish the ground truth, we utilize the open-weight Qwen3-32B model as our verification oracle. 
As we have full access to its internal distributions, we can analytically compute the exact, true trajectory log-probabilities ($\pi_{teacher}^*$) for any given student reasoning chunk. 
We then measure how closely our black-box approximations, the frequentist ($\hat{\pi}_{freq}$) and Bayesian ($\hat{\pi}_{teacher}$) path estimators, align with this true probability density. 
\tabref{tab:prob_estimation} presents the Mean Squared Error (MSE) and Pearson Correlation across various semantic metrics ($\phi$) and Monte Carlo rollout budgets ($N$).
\begin{table}[tb]
\centering
\caption{Comparison of estimation accuracy between Frequentist and Bayesian estimators across different semantic similarity metrics. Evaluated over $1000$ samples with chunk size $C=50$, number of chunks $M = 20$, and prior strength $\alpha=1.0$. The Bayesian estimator consistently reduces MSE and improves correlation with the ground truth.}
\label{tab:prob_estimation}
\vspace{-0.1in}
\setlength{\tabcolsep}{10pt}
\resizebox{\textwidth}{!}{%
\begin{tabular}{lccccc}
\toprule
\textbf{Semantic Metric} & \textbf{Rollouts ($N$)} & \textbf{MSE (Freq) $\downarrow$} & \textbf{MSE (Bayes) $\downarrow$} & \textbf{Corr (Freq) $\uparrow$} & \textbf{Corr (Bayes) $\uparrow$} \\ 
\midrule
\multirow{2}{*}{\texttt{ROUGE-1} \citep{lin2004rouge}} & 10 & $0.0279$ & $\mathbf{0.0217}$ & $0.5530$ & $0.6042$ \\
 & 20 & $0.0274$ & $0.0241$ & $0.5569$ & $0.5838$ \\ 
\midrule
\multirow{2}{*}{\texttt{ROUGE-L} \citep{lin2004rouge}} & 10 & $0.0667$ & $0.0521$ & $0.6396$ & $0.6875$ \\
 & 20 & $0.0663$ & $0.0584$ & $0.6466$ & $0.6726$ \\ 
\midrule
\multirow{2}{*}{\texttt{Jaccard} \citep{niwattanakul2013using}} & 10 & $0.0827$ & $0.0650$ & $0.5865$ & $0.6464$ \\
 & 20 & $0.0826$ & $0.0730$ & $0.5885$ & $0.6210$ \\ 
\midrule
\multirow{2}{*}{\texttt{BLEU-1} \citep{papineni2002bleu}} & 10 & $0.0279$ & $0.0217$ & $0.5539$ & $0.6048$ \\
 & 20 & $0.0273$ & $0.0240$ & $0.5571$ & $0.5842$ \\ 
\midrule
\multirow{2}{*}{\texttt{BLEU-2} \citep{papineni2002bleu}} & 10 & $0.1206$ & $0.0952$ & $0.6402$ & $0.6835$ \\
 & 20 & $0.1202$ & $0.1066$ & $0.6476$ & $0.6707$ \\ 
\midrule
\multirow{2}{*}{\texttt{Edit distance} \citep{navarro2001guided}} & 10 & $0.1857$ & $0.1478$ & $0.6625$ & $\mathbf{0.7215}$ \\
 & 20 & $0.1851$ & $0.1647$ & $0.6713$ & $0.7033$ \\ 
\midrule
\multirow{2}{*}{\texttt{Exact match}} & 10 & $0.3009$ & $0.2417$ & $0.4625$ & $0.6994$ \\
 & 20 & $0.3006$ & $0.2688$ & $0.5051$ & $0.6468$ \\ 
\bottomrule
\end{tabular}%
}
\end{table}
The results illustrate two critical dynamics within the \ours framework. 
First, the Bayesian estimator dominates the frequentist baseline across every evaluated metric and sample size. 
By stabilizing the sparse teacher signals with the student's normalized prior, the Bayesian update systematically reduces the MSE and yields a stronger positive correlation with the true probability density. 
Notably, under \texttt{Exact Match} at $N=10$, correlation improves from $0.46$ (frequentist) to $0.70$ (Bayesian).
This empirically validates \theoremref{thm:gradient_stability}(c), demonstrating that the Bayesian prior effectively dampens estimator variance during black-box optimization and provides a non-vanishing supervision lower bound at chunks where teacher rollouts are sparse (\theoremref{thm:gradient_stability}(b)).
These estimator-level improvements operate within a chunk-level loss formulation that, by design, eliminates the unbounded reverse-KL gradients of standard OPD (\theoremref{thm:gradient_stability}(a)).

Second, the data reveals a clear trade-off between magnitude accuracy and rank correlation. 
Rigid criteria like \texttt{Exact Match} suffer from severe sparsity, producing the highest estimation errors (MSE $>0.24$). 
Strict structural metrics like \texttt{Edit Distance} also underestimate absolute magnitudes (MSE $\approx 0.148$) but achieve the highest Pearson correlation ($0.72$), exceptionally preserving the relative ranking of reasoning paths. 
Conversely, unigram metrics like \texttt{ROUGE-1} bridge the tokenization gap to yield the most accurate absolute estimations, achieving the lowest MSE ($\approx 0.022$). 
We further systematically compare the empirical impact of these two metrics in our sensitivity analysis (\secref{sec:sensitivity_analysis}).

\subsection{Sensitivity Analysis}
\label{sec:sensitivity_analysis}
\begin{table}[t]
\centering
\caption{
    Sensitivity analysis of \ours hyperparameters: verification budget ($M$), Monte Carlo rollouts ($N$), and chunk size ($C$). 
    Experiments utilize the Qwen3-4B student and Qwen3-32B teacher setup. 
    Best results are in \textbf{bold}.
}
\vspace{-0.1in}
\label{tab:sensitivity_hyperparams}
\setlength{\tabcolsep}{18pt}
\resizebox{\textwidth}{!}{%
\begin{tabular}{ccccccccc}
\toprule
$M$ & $N$ & $C$ & \textbf{MATH} & \textbf{AIME-24} & \textbf{AIME-25} & \textbf{AMC-23} & \textbf{Olympiad} & \textbf{Avg} \\
\midrule
10 & 10 & 50 & 72.60 & 68.67 & 60.42 & 93.44 & 50.27 & 69.08 \\
\midrule
5 & 10 & 50 & 72.76 & 69.58 & 55.63 & 93.44 & 50.48 & 68.38 \\
10 & 10 & 25 & 52.73 & 10.42 & 4.79 & 32.19 & 22.26 & 24.48 \\
\midrule
10 & 20 & 50 & 73.35 & 68.54 & 61.04 & 93.13 & 51.81 & 69.57 \\
20 & 10 & 50 & 75.01 & 70.63 & \textbf{62.71} & 90.94 & 55.50 & 70.96 \\
10 & 10 & 100 & \textbf{75.94} & \textbf{72.29} & 62.50 & \textbf{93.59} & 53.58 & \textbf{71.58} \\
\midrule
10 & 20 & 100 & 75.00 & 72.08 & 61.88 & 93.13 & \textbf{55.72} & 71.56 \\
\bottomrule
\end{tabular}%
}
\end{table}
We analyze \ours's sensitivity to its five hyperparameters --- chunk number ($M$), rollout count ($N$), chunk size ($C$), Bayesian prior ($\alpha$), and the semantic metric ($\phi$) --- on the Qwen3-4B (student) and Qwen3-32B (teacher) setup. 
Results are shown in \tabref{tab:sensitivity_hyperparams}--\ref{tab:sensitivity_semantic_metric} and \figref{fig:alpha_sensitivity}. 
% }
% The results, utilizing the Qwen3-4B student and Qwen3-32B teacher, are detailed in \tabref{tab:sensitivity_hyperparams} and \tabref{tab:sensitivity_semantic_metric}.

\paragraph{The criticality of chunk size ($C$).}
The most pronounced sensitivity lies in the chunk size. 
Expanding the chunk size to $C=100$ yields the highest performance (71.58\%), indicating that larger chunks give the Bayesian estimator more structural context per audit.
Conversely, halving the chunk size to $C=25$ causes a dramatic degradation from 69.08\% to 24.48\%. 
This confirms our hypothesis that a 25-token window is too narrow to encapsulate a meaningful logical deduction. 
At such scales, the semantic similarity metric $\phi$ becomes too noisy, penalizing the student for minor phrasing differences rather than verifying actual reasoning progress. 

\paragraph{Scaling Monte Carlo rollouts ($N$).}
Increasing the number of teacher rollouts from $N=10$ to $20$ yields only marginal improvements.
Under a chunk size of $C=50$, $N=20$ provides a slight +0.49\% gain over the $N=10$ baseline. 
Moreover, under the optimal $C=100$ setting, scaling to $N=20$ actually results in a fractionally lower average. 
This diminishing-returns pattern matches \theoremref{thm:concentration}, which establishes that the estimator's accuracy improves at rate $O(1/N)$, so the marginal benefit of doubling $N$ halves with each doubling.

% The expanded sample size does not meaningfully tighten the Bayesian estimation enough to justify doubling the API query volume. 
% %
% This consistently demonstrates that $N=10$ sits at an optimal Pareto frontier, capturing sufficient semantic variance for robust distillation while remaining highly budget-efficient.

\paragraph{Robustness to chunk number ($M$).}
\ours allows for flexible scaling of the total verification budget. 
Increasing the audited chunks to $M=20$ provides a solid performance boost, raising the average to 70.96\% by granting the student denser supervisory signals across the trajectory. 
Conversely, the framework is also highly resilient to aggressive budget reductions.
Halving the number of audited chunks from $M=10$ to $M=5$ results in only a minimal performance dip (69.08\% down to 68.38\%). 
This resilience serves as a strong empirical validation for our peak-entropy anchored scheduling (\secref{subsubsec:speculative_verification}). 
Because the $M=5$ chunks are strictly targeted at the most critical, high-uncertainty decision boundaries, the corrective gradient signal remains highly potent, allowing for massive teacher cost reductions without sacrificing the student's final reasoning capability. 
A detailed per-sample cost analysis demonstrating the Pareto efficiency and economic scalability of these hyperparameter configurations is provided in \appref{app:cost_analysis}.
\begin{table}[t]
\centering
\caption{Sensitivity analysis of the semantic measurement function ($\phi$). We compare normalized Edit Distance against ROUGE-1 unigram overlap across different student-teacher pairings. Best results per setup are in \textbf{bold}.}
\label{tab:sensitivity_semantic_metric}
\setlength{\tabcolsep}{10pt}
\resizebox{\columnwidth}{!}{%
\begin{tabular}{lllcccccc}
\toprule
\textbf{Student Model} & \textbf{Teacher Model} & \textbf{Metric ($\phi$)} & \textbf{MATH} & \textbf{AIME-24} & \textbf{AIME-25} & \textbf{AMC-23} & \textbf{Olympiad} & \textbf{Avg} \\
\midrule
\multirow{2}{*}{Qwen3-1.7B} & \multirow{2}{*}{Qwen3-32B} & Edit Distance & 84.61 & 47.08 & 33.33 & 79.68 & \textbf{58.64} & 60.67 \\
 & & ROUGE-1 & \textbf{86.30} & \textbf{48.13} & \textbf{36.46} & \textbf{85.31} & 58.16 & \textbf{62.87} \\
\midrule
\multirow{2}{*}{Qwen3-4B} & \multirow{2}{*}{Qwen3-32B} & Edit Distance & \textbf{72.60} & 68.67 & 60.42 & \textbf{93.44} & 50.27 & 69.08 \\
 & & ROUGE-1 & 70.36 & \textbf{72.71} & \textbf{61.04} & 92.18 & \textbf{50.52} & \textbf{69.36} \\
\midrule
\multirow{2}{*}{Qwen3-4B} & \multirow{2}{*}{Qwen3-30B-A3B} & Edit Distance & \textbf{74.98} & \textbf{75.00} & \textbf{62.92} & 92.65 & \textbf{56.05} & \textbf{72.32} \\
 & & ROUGE-1 & 74.78 & 71.67 & 61.88 & \textbf{93.75} & 55.72 & 71.56 \\
\bottomrule
\end{tabular}%
}
\end{table}

\begin{wrapfigure}{r}{0.5\textwidth}
\vspace{-0.25in}
    \centering
    \includegraphics[width=\linewidth]{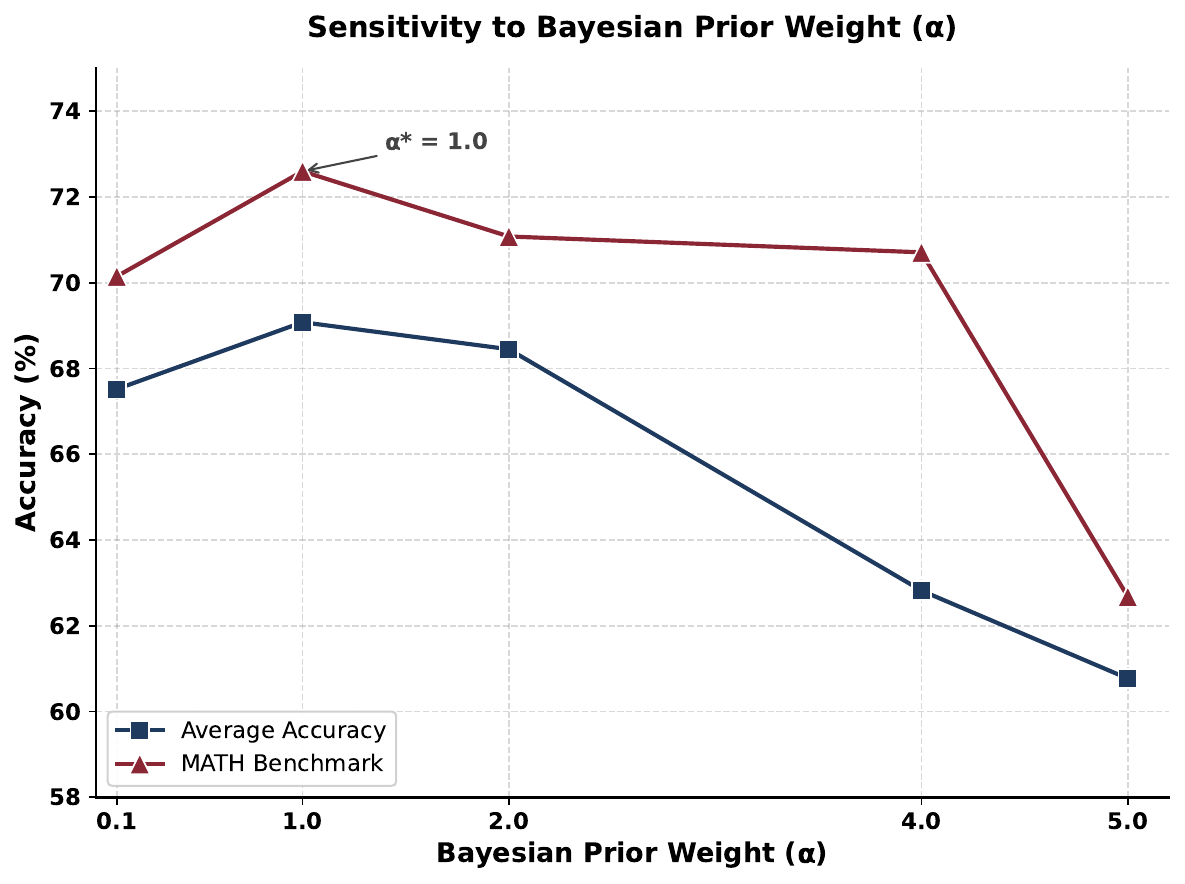}
    \vspace{-0.25in}
    \caption{
        Sensitivity of \ours to the Bayesian prior ($\alpha$). 
        $\alpha=1.0$ achieves the optimal balance between smoothing Monte Carlo variance and allowing policy updates.
    }
    \label{fig:alpha_sensitivity}
    \vspace{-0.15in}
\end{wrapfigure}

\paragraph{Impact of the Bayesian prior ($\alpha$).}
The parameter $\alpha$ controls the strength of the base-model prior in our Bayesian smoothing step. 
As illustrated in \figref{fig:alpha_sensitivity}, \ours exhibits a distinct performance peak at $\alpha=1.0$ (69.08\% average). 
Deviating from this optimal balance in either direction results in performance degradation, illustrating the bias-variance tradeoff inherent in discrete sampling. 
When the prior is too weak ($\alpha=0.1$), the average accuracy drops to 67.52\%; without sufficient smoothing, the framework overfits to the severe statistical variance and zero-probability traps of the $N=10$ discrete Monte Carlo rollouts. 
Conversely, applying an overly aggressive prior ($\alpha \geq 4.0$) causes a steep performance collapse, plummeting to 60.77\% at $\alpha=5.0$. 
In this high-$\alpha$ regime, the trust-region penalty becomes too restrictive, forcing the student to stubbornly anchor to its own sub-optimal base policy and effectively blocking it from learning the teacher's trajectories.

\paragraph{Robustness to semantic measurement functions ($\phi$).}
\ours relies on a continuous measurement function $\phi$ to evaluate the student's chunk-level progress. 
\tabref{tab:sensitivity_semantic_metric} compares our default normalized \texttt{Edit Distance} against standard \texttt{ROUGE-1} unigram overlap across three student--teacher pairings. 
The results demonstrate that \ours is broadly robust to the choice of $\phi$, although the optimal metric depends on the student--teacher capability gap and the teacher's alignment level.
%
% The empirical results demonstrate that \ours is highly adaptable to the chosen string-matching algorithm, though the optimal metric is heavily dependent on the student's capacity and the teacher's alignment level.
When distilling the 32B model into the highly constrained 1.7B student, \texttt{ROUGE-1} provides a significant advantage (+2.20\% average), even better than the 1.7B GRPO baseline in \tabref{tab:main_results_math}. 
Because the smaller student struggles to perfectly mimic the complex structural phrasing of the 32B teacher, \texttt{Edit Distance} is overly punitive. 
\texttt{ROUGE-1} grants the 1.7B model the necessary linguistic flexibility to achieve the correct reasoning milestones. 
As the student capacity increases to 4B, the gap between the metrics narrows to a marginal +0.28\% advantage for \texttt{ROUGE-1}.
Conversely, when distilling from the heavily aligned Qwen3-30B-Instruct model into the 4B student, \texttt{Edit Distance} secures a slight edge (+0.76\% average). 
This variance suggests that while \texttt{Edit Distance} enforces the strict relative ranking constraints required for mimicking highly aligned, structured reasoning, \texttt{ROUGE-1} offers the vocabulary flexibility required for extreme capability jumps between mismatched base models.

\subsection{Ablation Studies}
\label{sec:ablation_studies}

To isolate the empirical contribution of each core mechanism in \ours, we conduct an ablation study using the Qwen3-4B student and Qwen3-32B teacher setup. 
The results, presented in \tabref{tab:ablation}, validate the theoretical properties established in \secref{sec:theoretical_analysis}. 

The most dramatic impact is observed when removing the base-model KL anchor ($\beta = 0$). 
Without this regularization, the student's average performance catastrophically collapses from 69.08\% to 8.28\%. 
This corroborates \theoremref{theorem:trust_region}.
Because \ours audits only $M$ discrete chunks and leaves vast unsupervised gaps in the generated trajectory, the absence of a KL constraint allows the policy to degenerate rapidly, with the student exploiting the sparse verification signals at the cost of foundational linguistic fluency and structural reasoning coherence. 
%
% The student exploits the sparse verification signals, ultimately losing its foundational linguistic fluency and structural reasoning coherence. 

Beyond the necessity of the trust region, we also validate the impact of our statistical and scheduling mechanisms. 
Replacing the Bayesian proxy with raw frequentist estimation (i.e., raw semantic match ratios without the base-model prior) degrades the overall average accuracy from 69.08\% to 68.63\%. 
While the model maintains baseline functionality, the lack of Bayesian smoothing exposes the optimization process to the severe estimation error inherent in limited sampling (\theoremref{thm:gradient_stability}). 
Finally, we substitute our peak-entropy anchored selection with a naive uniform sampling strategy, which further reduces the overall average to 68.45\%. 
By blindly evaluating segments of the reasoning path, the uniform approach wastes our highly constrained verification budget on deterministic, low-complexity transitions. 
Anchoring the verification chunks strictly at points of maximum predictive entropy ensures the black-box teacher's corrective signal is concentrated exactly where the student policy faces critical, uncertain decision boundaries. 
To provide concrete intuition for how this peak-entropy scheduling, we also provide detailed qualitative case studies in \appref{sec:appendix_case_studies}.

\begin{table}[t]
\centering
\caption{Ablation study of \ours components. The setup utilizes the Qwen3-4B student and the Qwen3-32B teacher. Removing the KL anchoring results in catastrophic policy collapse, validating our theoretical bounds.}
\label{tab:ablation}
\vspace{-0.1in}
\setlength{\tabcolsep}{18pt}
\resizebox{\textwidth}{!}{%
\begin{tabular}{lcccccc}
\toprule
\textbf{Variant} & \textbf{MATH} & \textbf{AIME-24} & \textbf{AIME-25} & \textbf{AMC-23} & \textbf{Olympiad} & \textbf{Avg} \\
\midrule
\ours (Full) & 72.60 & 68.67 & \textbf{60.42} & \textbf{93.44} & \textbf{50.27} & \textbf{69.08} \\
\midrule
w/o Entropy Selection & 70.35 & \textbf{70.42} & 58.54 & 92.69 & 50.26 & 68.45 \\
w/o Bayesian Smoothing & \textbf{72.95} & 70.00 & 58.29 & 91.87 & 50.04 & 68.63 \\
w/o KL Anchoring & 9.27 & 8.13 & 6.25 & 12.03 & 5.71 & 8.28 \\
\bottomrule
\end{tabular}%
}
\end{table}
\section{Conclusion}
\label{sec:conclusion}

In this work, we introduced \ours, a novel logit-free, chunk-level OPD framework that addresses both the white-box access requirement of standard OPD and the brittleness of its token-level supervision signal. 
By unifying peak-entropy chunk scheduling, Bayesian-smoothed speculative verification, and base-model KL anchoring, our approach efficiently extracts robust, variance-reduced gradient signals from teacher models while mathematically bounding policy degeneration. 
Extensive evaluations demonstrate that \ours not only matches the performance of standard OPD but also scales seamlessly to proprietary models. 
Ultimately, \ours bridges the closed-source capability divide, empowering smaller open-weight models to harness the reasoning capacity of frontier models and comprehensively surpass the traditional limits of offline fine-tuning and self-exploratory reinforcement learning.

\clearpage
\newpage
\bibliographystyle{assets/plainnat}
\bibliography{custom}

\clearpage
\newpage
\beginappendix
% \appendix
% \section*{Appendix}

\section{Prompt Templates}
\label{sec:appendix_prompts}

To ensure consistency between the learning and testing distributions, we apply the exact same prompt template across both the training phase (DAPO-Math-17K) and all evaluation benchmarks. As shown in \tabref{tab:prompt_templates}, this template includes a specific directive to format the final answer, ensuring reliable automated parsing via regular expressions.

\begin{table}[h!]
\centering
\caption{Prompt template utilized for all mathematical reasoning tasks (both training and evaluation). The \texttt{\{Question\}} placeholder is replaced with the specific problem text during inference.}
\label{tab:prompt_templates}
\renewcommand{\arraystretch}{1.5}
\begin{tabular}{p{0.2\textwidth} p{0.75\textwidth}}
\toprule
\textbf{Application} & \textbf{Prompt Template} \\
\midrule
\textbf{Training \& \newline Evaluation} & 
\texttt{Solve the following math problem step by step. The last line of your response should be of the form Answer: \textbackslash\$Answer (without quotes) where \textbackslash\$Answer is the answer to the problem.\textbackslash n\textbackslash n\{Question\}} \\
\bottomrule
\end{tabular}
\end{table}

\section{Viability and Teacher Inference Analysis}
\label{app:cost_analysis}
% \zz{
The cost of any on-policy distillation method is teacher inference --- specifically, the number of tokens the teacher must process and generate per training sample. 
This cost dominates both self-hosted teacher deployments, where it determines GPU-hours and wall-clock training time, and hosted-teacher deployments.
In both regimes, modern transformer inference exhibits a sharp asymmetry between its two phases. 
Prefill (consuming input tokens) is parallelizable across the sequence and compute-bound, whereas decode (generating output tokens) is sequential and memory-bandwidth-bound.
Across typical hardware configurations and batching regimes, this yields roughly an order-of-magnitude per-token gap, with decode being the substantially more demanding of the two.

Standard offline SFT places its entire supervision cost on the decode side, having the teacher generate the full reasoning trajectory (~8,000 decode tokens per sample in DAPO-Math-17K) with negligible prefill. 
\ours instead exploits the prefill/decode asymmetry by passing the student's existing trajectory as context --- the teacher absorbs most of its work in cheap parallel prefill and emits only short, targeted decode bursts at the audited chunks.

The per-sample token consumption decomposes across the two phases. 
On the decode side, at each of the $M$ selected chunks the teacher produces $N$ independent Monte Carlo rollouts of length $C$, yielding $M \times N \times C$ decode tokens per sample. 
For $M = 10, N = 20, C = 50$, this is exactly $10,000$ decode tokens. 
On the prefill side, at each chunk the teacher is conditioned on the original task prompt plus the student's trajectory prefix up to the forking point. 
Assuming the $M$ anchor chunks are roughly uniformly distributed across a student trajectory of length $L$, the expected prefix length is $L/2$, giving $M \times (\text{Prompt Length} + L/2)$ prefill tokens per sample.
For $M = 10, L \approx 6,340$, and a 200-token prompt, this totals around $33,700$ prefill tokens.

To compare configurations on a common footing, \tabref{tab:cost_analysis} reports each configuration's total teacher inference cost as a multiplier relative to standard SFT, computed under a representative 8.3:1 decode-to-prefill per-token ratio. 
This ratio is consistent both with the prefill/decode compute asymmetry of modern transformer inference and with the average effort of frontier hosted teachers, and the qualitative conclusions are insensitive to its exact value across the 5x–10x range typical of current systems.

Under the default configuration ($M = 10, C = 50$), \ours costs only 1.75× standard SFT while providing dense on-policy correction at every audited chunk — a modest overhead given that on-policy supervision otherwise typically requires full-trajectory teacher re-generation at every training step. 
More strikingly, the sparse configuration ($M = 5$), which retains 99\% of the baseline accuracy (68.38\% vs 69.08\%), comes in at 0.88x SFT, making it strictly more efficient than offline distillation. 
At the other end of the spectrum, the highest-accuracy configuration ($C = 100$) scales to 3.00× SFT. 
This range traces a controllable effort–accuracy frontier, allowing \ours to be tuned to widely varying training budgets without altering its core mechanism.
\begin{table*}[t]
\centering
\caption{
    Per-sample teacher inference cost across \ours configurations. 
    Token counts are exact; the relative effort column normalizes total teacher inference work against standard SFT under a representative 8.3:1 decode-to-prefill per-token effort ratio, consistent with the prefill/decode asymmetry of modern transformer inference. 
    The sparse configuration ($M = 5$) is strictly more efficient than offline SFT while retaining 99\% of the full-configuration accuracy, and the framework scales smoothly to higher-accuracy regimes at bounded multipliers.
    %
    % Comprehensive cost analysis per mathematical reasoning sample using Gemini 2.5 Flash. 
    % %
    % By shifting the generation burden to heavily discounted input tokens, \ours is highly cost-competitive. 
    % %
    % Notably, when using the sparse configuration ($M=5$), \ours is even cheaper than standard SFT, while higher-budget configurations offer bounded performance scaling.
}
\label{tab:cost_analysis}
\setlength{\tabcolsep}{8pt}
\resizebox{\textwidth}{!}{%
\begin{tabular}{lrrrr}
\toprule
\textbf{Method} 
& \textbf{Hyperparameters} 
& \textbf{Teacher Input} 
& \textbf{Teacher Output} 
& \textbf{Relative Effort}
% & \textbf{Input Cost} 
% & \textbf{Output Cost} 
% & \textbf{Total Cost} 
\\ 
\midrule
Standard SFT & --- & $200$ & $8,000$ &1.00x
% & $\$0.0001$ & $\$0.0200$ & $\$0.020$ ($1.00\times$) 
\\
\midrule
\multirow{7}{*}{\ours} 
 & $M=5, N=10, C=50$ & $16,850$ & $5,000$ &0.88x
 % $\$0.0051$ & $\$0.0125$ & $\mathbf{\$0.018}$ & \textbf{(0.88$\mathbf{\times}$)} 
 \\
 & $M=10, N=10, C=25$ & $33,700$ & $5,000$ &1.13x
 % & $\$0.0101$ & $\$0.0125$ & $\$0.023$ ($1.13\times$) 
 \\
 & $M=10, N=10, C=50$ & $33,700$ & $10,000$ &1.75x
 % & $\$0.0101$ & $\$0.0250$ & $\$0.035$ ($1.75\times$) 
 \\
 & $M=10, N=10, C=100$ & $33,700$ & $20,000$ &3.00x
 % & $\$0.0101$ & $\$0.0500$ & $\$0.060$ ($3.00\times$) 
 \\
 & $M=10, N=20, C=50$ & $67,400$ & $20,000$ &3.50x
 % & $\$0.0202$ & $\$0.0500$ & $\$0.070$ ($3.50\times$) 
 \\
 & $M=20, N=10, C=50$ & $67,400$ & $20,000$ &3.50x
 % & $\$0.0202$ & $\$0.0500$ & $\$0.070$ ($3.50\times$) 
 \\
 & $M=10, N=20, C=100$ & $67,400$ & $40,000$ &6.00x
 % & $\$0.0202$ & $\$0.1000$ & $\$0.120$ ($6.00\times$) 
 \\
\bottomrule
\end{tabular}}
\end{table*}

\clearpage
\section{Supplementary Details of \ours}
\label{app:supp_omniopd}

\subsection{Full Algorithm}
\begin{algorithm}[H]
\caption{The \ours Optimization Framework}\label{alg:ours_framework}
\label{alg:OPD}
\textbf{Input:} Base reference policy $\pi_{ref}$, Student policy $\pi_\theta$, Black-box teacher $\pi_{teacher}^*$ \\
\textbf{Hyperparameters:} chunk number $M$, chunk size $C$, Monte Carlo rollouts $N$, prior strength $\alpha$, KL penalty weight $\beta$, learning rate $\eta$ \\
\textbf{Output:} Optimized student policy $\pi_\theta$
\begin{algorithmic}[1]
\WHILE{not converged}
    \STATE Sample instruction $x \sim \mathcal{D}$
    \STATE \textbf{// 1. On-Policy Trajectory Generation}
    \STATE Generate student reasoning trajectory $y = (y_1, \dots, y_T) \sim \pi_\theta(\cdot \mid x)$
    \STATE \textbf{// 2. Peak-Entropy Anchored Scheduling}
    \FOR{$t = 1$ to $T$}
        \STATE Compute vocabulary entropy: $\mathcal{H}_t = - \sum_{v \in V} \pi_\theta(v \mid x, y_{<t}) \log \pi_\theta(v \mid x, y_{<t})$
    \ENDFOR
    \STATE Select $M$ anchor tokens: $A = \mathop{\arg\max}\limits_{t \subset \{1 \dots T\}, |A|=M} (\mathcal{H}_t)$
    \STATE Extract chunk set $\mathcal{C} = \{c_1, \dots, c_M\}$ based on anchors $A$
    \STATE Identify un-audited tokens: $U = \{1 \dots T\} \setminus \bigcup_{c \in \mathcal{C}} c$
    
    \STATE \textbf{// 3. Speculative Verification \& Bayesian Smoothing}
    \FOR{each chunk $c \in \mathcal{C}$}
        \STATE Query black-box teacher $\pi_{teacher}^*$ with prefix $y_{<c}$ for $N$ rollouts $\{y_{teacher}^{(i)}\}_{i=1}^N$
        \STATE Compute semantic match evidence: $k_{sem}^{(c)} = \sum_{i=1}^{N} \phi(y_c, y_{teacher}^{(i)})$
        \STATE Compute normalized base prior: $\bar{\pi}_\theta^{(c)} = \left( \prod_{t \in c} \pi_\theta(y_t \mid x, y_{<t}) \right)^{\frac{1}{C}}$
        \STATE Estimate Bayesian target proxy: $\hat{\pi}_{teacher}^{(c)} = \frac{k_{sem}^{(c)} + \alpha \cdot \bar{\pi}_\theta^{(c)}}{N + \alpha}$
    \ENDFOR
    
    \STATE \textbf{// 4. Trust Region Anchoring \& Optimization}
    \STATE Calculate chunk-level policy loss:
    \[ \mathcal{L}_{chunk} = - \sum_{c \in \mathcal{C}} \left( \hat{\pi}_{teacher}^{(c)} \sum_{t \in c} \log \pi_\theta(y_t \mid x, y_{<t}) \right) \]
    \STATE Calculate un-audited trust region penalty:
    \[ \mathcal{L}_{KL} = \beta \sum_{t \in U} D_{KL} \Big(\pi_{ref}(\cdot \mid x, y_{<t}) \parallel \pi_\theta(\cdot \mid x, y_{<t})\Big) \]
    \STATE Compute total objective: $\mathcal{L}_{\ours}(\theta) = \mathcal{L}_{chunk} + \mathcal{L}_{KL}$
    \STATE Perform gradient descent step: $\theta \leftarrow \theta - \eta \nabla_\theta \mathcal{L}_{\ours}(\theta)$
\ENDWHILE
\RETURN $\pi_\theta$
\end{algorithmic}
\end{algorithm}

\clearpage
\subsection{Training Dynamics}
\label{sec:training_dynamics}

To understand the optimization stability and learning progression of \ours, we visualize the training dynamics of the Qwen3-4B student in \figref{fig:training_dynamics}. 
We track the process under both an open-weight Qwen3-32B and Gemini-2.5-flash. 
The chunk-level on-policy loss exhibits rapid and stable convergence across both setups. 
For the Qwen3-32B teacher, the loss drops sharply from approximately 0.33 to 0.24 within the first 100 steps before stabilizing. 
Remarkably, when utilizing the Gemini-2.5-flash teacher, the student achieves a similar loss plateau in merely 30 steps. 
This steady convergence empirically validates our Bayesian smoothing mechanism, which anchors the target probability to prevent gradient explosions and ensure stable parameter updates across the reasoning trajectory.

As the student rapidly deviates from its sub-optimal base behavior to align with the teacher's superior reasoning structures, we observe a corresponding initial rise in the KL divergence between the active policy and the frozen base model. 
Crucially, rather than growing unboundedly, the KL curves safely plateau. This trajectory perfectly mirrors the theoretical bounds established in Theorem \ref{theorem:trust_region}, confirming that the trust-region penalty successfully constrains the policy shift within the un-audited gaps and avoids catastrophic entropy collapse.

This stable optimization directly translates to continuous progression on downstream evaluations. 
Across the training steps, the pass-rate accuracy on the challenging AIME-2025 dataset climbs substantially alongside the reduction in on-policy loss (note that the performance gap between this validation curve and the final results in \tabref{tab:main_results_math} is due to maximum token truncation during intermediate evaluation). 
Under the Qwen3-32B teacher, AIME accuracy rises from a sub-30\% baseline to peak above 50\% near step 240. Furthermore, the Gemini-2.5-flash teacher drives the student from 23\% to roughly 47\% accuracy within a strictly constrained 50-step budget. This rapid escalation proves that our peak-entropy chunk scheduling is fundamentally sound: by exclusively correcting the reasoning milestones at points of highest predictive uncertainty, the localized gradient signals efficiently generalize to holistic mathematical problem-solving.
\begin{figure*}[t]
    \centering
    % Top row: Qwen3-32B Teacher
    \begin{subfigure}{0.32\textwidth}
        \includegraphics[width=\linewidth]{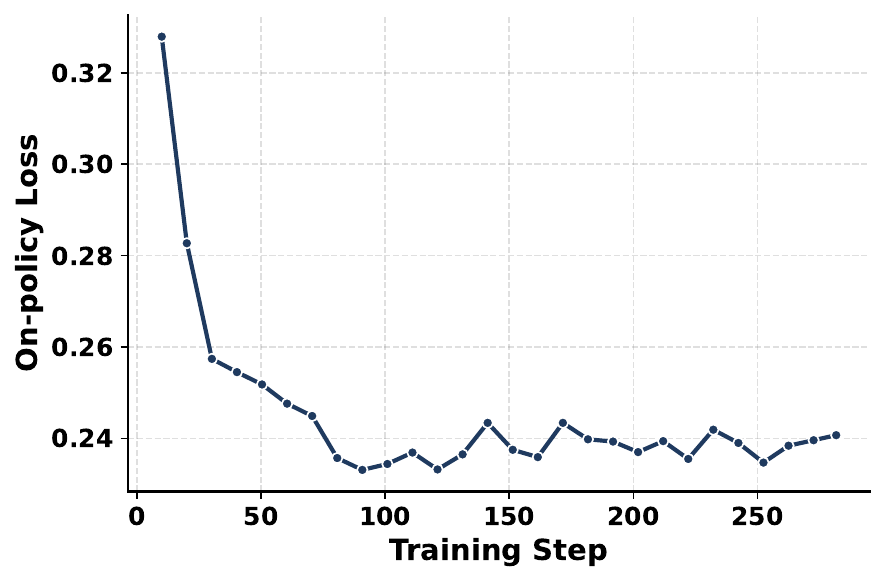}
        \caption{On-Policy Loss (Qwen3-32B)}
        \label{fig:dyn_loss_qwen}
    \end{subfigure}
    \hfill
    \begin{subfigure}{0.32\textwidth}
        \includegraphics[width=\linewidth]{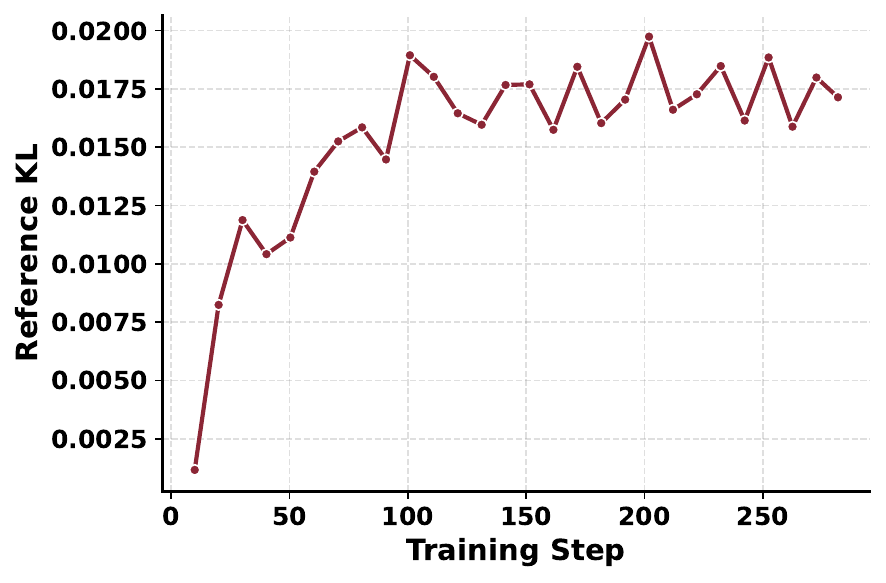}
        \caption{Reference KL (Qwen3-32B)}
        \label{fig:dyn_kl_qwen}
    \end{subfigure}
    \hfill
    \begin{subfigure}{0.32\textwidth}
        \includegraphics[width=\linewidth]{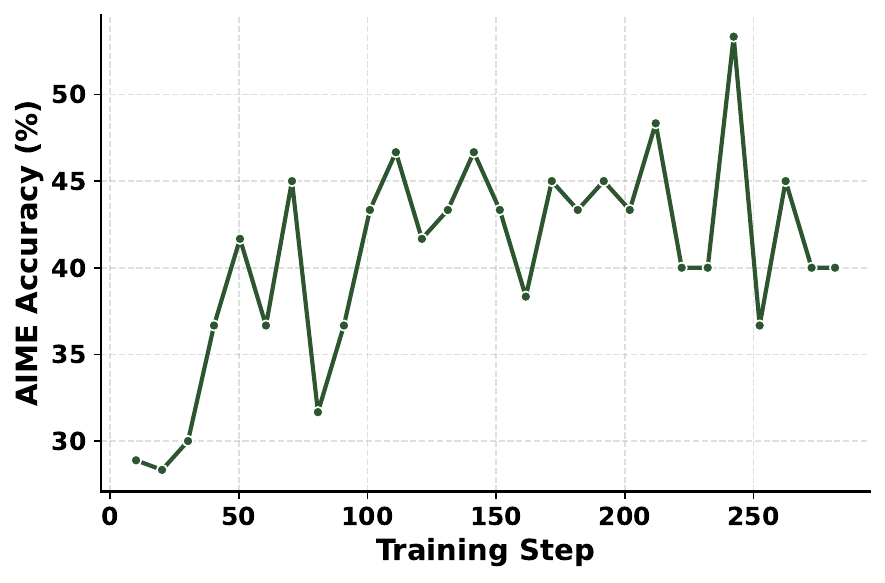}
        \caption{AIME Accuracy (Qwen3-32B)}
        \label{fig:dyn_acc_qwen}
    \end{subfigure}
    
    \vspace{0.3cm} % Vertical spacing between rows
    
    % Bottom row: Gemini-2.5-flash Teacher
    \begin{subfigure}{0.32\textwidth}
        \includegraphics[width=\linewidth]{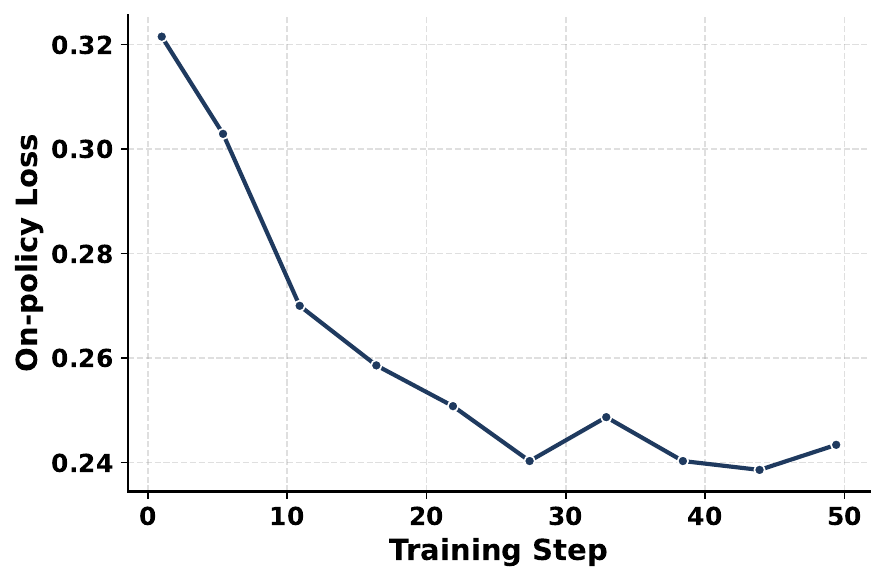}
        \caption{On-Policy Loss (Gemini-2.5-Flash)}
        \label{fig:dyn_loss_gemini}
    \end{subfigure}
    \hfill
    \begin{subfigure}{0.32\textwidth}
        \includegraphics[width=\linewidth]{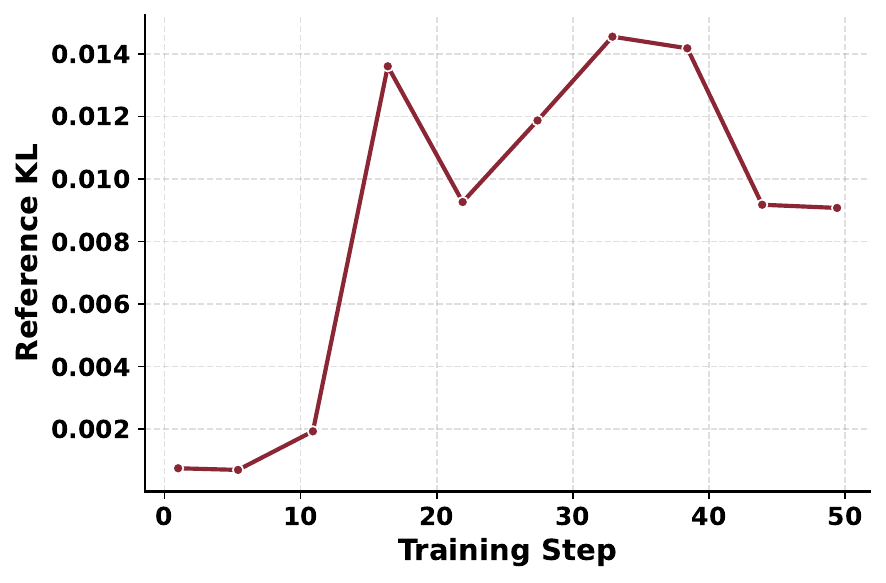}
        \caption{Reference KL (Gemini-2.5-Flash)}
        \label{fig:dyn_kl_gemini}
    \end{subfigure}
    \hfill
    \begin{subfigure}{0.32\textwidth}
        \includegraphics[width=\linewidth]{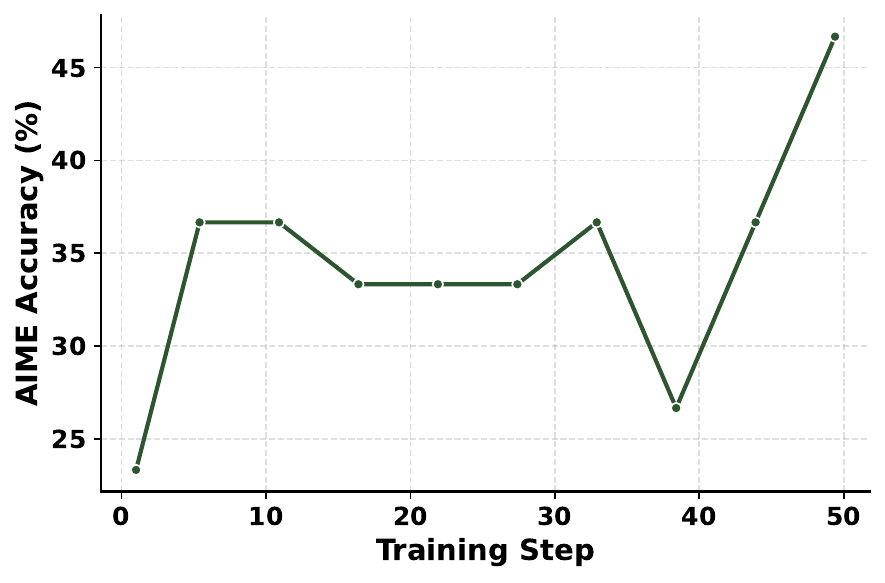}
        \caption{AIME Accuracy (Gemini-2.5-Flash)}
        \label{fig:dyn_acc_gemini}
    \end{subfigure}
    
    \caption{Training dynamics of the Qwen3-4B student distilled via \ours. The top row utilizes the open-weight Qwen3-32B teacher, while the bottom row utilizes the frontier Gemini-2.5-flash model. We track the on-policy loss (left), the base-model KL divergence penalty (center), and the continuous improvement in AIME 2025 evaluation accuracy (right) across training steps.}
    \label{fig:training_dynamics}
\end{figure*}

\section{Supplementary Results}
\subsection{Case Studies: Qualitative Analysis of Entropy-Selected Chunks}
\label{sec:appendix_case_studies}

This section presents how \ours isolates critical reasoning junctures. Rather than auditing the entire trajectory, our peak-entropy scheduling concentrates the teacher's verification budget on ``forks in the road'' where the student's policy exhibits high uncertainty.

\paragraph{Pattern Summarization: Why Entropy Works}
By analyzing the selected chunks across various domains, we identify three recurring patterns that justify the effectiveness of entropy-based selection:
\begin{itemize}
    \item \textbf{Strategic Transitions:} Spikes in entropy frequently occur when the model finishes a setup phase and must choose a specific theorem or tool (e.g., deciding between the \textit{Law of Cosines} vs. \textit{Area Formulas} in Geometry).
    \item \textbf{Boundary and Constraint Validation:} In algebraic tasks, uncertainty peaks at logical branching points, such as defining cases for absolute values or checking if a solution $x=1$ satisfies the initial domain constraints.
    \item \textbf{Combinatorial State Tracking:} In multi-step counting, entropy rises when the model must keep track of previously assigned states (e.g., vertex colors), representing a "memory bottleneck" where teacher correction is most vital.
\end{itemize}

In the following examples, we show a subset of the 10 selected chunks for three problems. The symbol \textbf{||} marks the start of the verification chunk.

\paragraph{Case 1: Geometry --- Angle Bisector}
\textit{Problem: Find AD in $\triangle ABC$ where $\sin A = 3/5$, $AB=6, AC=8$.}
\begin{itemize}
    \item \textbf{Fork \#2 (Pos: 212, Entropy: 1.16):} \texttt{... But I don't know BC yet. Maybe I can find BC using the Law of Cosines? But} \textbf{||} \texttt{wait, I know sin A = 3/5. Since angle A is acute, cos A should be positive...}
    \item \textbf{Fork \#7 (Pos: 915, Entropy: 1.21):} \texttt{...called the formula for the length of an angle bisector: AD\^{}2 = AB * AC - BD * DC.} \textbf{||} \texttt{Wait, let me check that. Let me recall Stewart's theorem. Stewart's theorem relates...}
    \item \textbf{Analysis:} The model exhibits high uncertainty exactly when it needs to verify its own memory of geometric theorems (Stewart's vs. Bisector theorem), ensuring the teacher corrects any "hallucinated" formulas.
\end{itemize}

\paragraph{Case 2: Combinatorics --- Hexagon Coloring}
\textit{Problem: Ways to color hexagon vertices with 3 colors (no adjacent same).}
\begin{itemize}
    \item \textbf{Fork \#2 (Pos: 514, Entropy: 1.69):} \texttt{...Wait, maybe not. Alternatively, using recurrence relations. But maybe I can} \textbf{||} \texttt{just trust the formula here. For n=4, the formula gives 18 for k=3. Let me check...}
    \item \textbf{Fork \#10 (Pos: 1383, Entropy: 1.50):} \texttt{...vertex 4 is not adjacent to vertex 2. So} \textbf{||} \texttt{if vertex 4 is B, is that okay? Since vertex 4 is adjacent to vertex 3 (A) and vertex 1...}
    \item \textbf{Analysis:} Entropy selection captures the model's struggle with cyclic constraints. It identifies the "wrap-around" check (vertex 6 vs vertex 1) as the highest-entropy region, which is the most common point of failure for smaller models.
\end{itemize}

\paragraph{Case 3: Algebra --- Absolute Value Equations}
\textit{Problem: Solve $|x^2 - 3x + 2| = x - 1$.}
\begin{itemize}
    \item \textbf{Fork \#3 (Pos: 395, Entropy: 1.56):} \texttt{... x\^{}2 - 4x + 3 = 0.} \textbf{||} \texttt{Factorizing this quadratic equation: Looking for two numbers that multiply to 3 and add to -4...}
    \item \textbf{Fork \#6 (Pos: 837, Entropy: 1.64):} \texttt{...absolute value is A when A is non-negative. Therefore,} \textbf{||} \texttt{the solutions from Case 1 must satisfy A >= 0. So, let's check for x = 1 and x = 3...}
    \item \textbf{Fork \#10 (Pos: 1381, Entropy: 1.58):} \texttt{...x = 1 is a solution to both equations. However,} \textbf{||} \texttt{we need to check if x = 1 is a solution to the original equation because...}
    \item \textbf{Analysis:} The selection focuses heavily on the \textit{verification} of solutions. While the calculation of $x^2-4x+3$ is relatively low-entropy, the \textit{logical validity} check (testing if the results fit Case 1 or Case 2) triggers a spike, allowing \ours to enforce rigorous case-split logic.
\end{itemize}

\end{document}